%% file: main.tex
\documentclass{article}

\usepackage{arxiv}

\usepackage[utf8]{inputenc} 
\usepackage[T1]{fontenc}    
\usepackage{hyperref}       
\usepackage{url}            
\usepackage{booktabs}       
\usepackage{amsfonts}       
\usepackage{nicefrac}       
\usepackage{microtype}      
\usepackage{lipsum}         
\usepackage{graphicx}

\renewcommand{\inf}[1]{\underset{#1}{\text{inf}}\,}

\newcommand{\argmin}[1]{\underset{#1}{\text{argmin}}\,}

\usepackage{algorithm}
\usepackage{algpseudocode}
\usepackage{colortbl}

\input{content/preamble}

\title{Dataset Dictionary Learning in a Wasserstein Space for Federated Domain Adaptation}

\date{}

\newif\ifuniqueAffiliation
\uniqueAffiliationtrue

\ifuniqueAffiliation 
\author{
Eduardo Fernandes Montesuma\\
CEA, List\\
Université Paris-Saclay\\
F-91120 Palaiseau, France
\And
Fabiola Espinoza Castellon\\
CEA, List\\
Université Paris-Saclay\\
F-91120 Palaiseau, France
\And
Fred Ngolè Mboula\\
CEA, List\\
Université Paris-Saclay\\
F-91120 Palaiseau, France
\And
Aur\'elien Mayoue\\
CEA, List\\
Université Paris-Saclay\\
F-91120 Palaiseau, France
\And
Antoine Souloumiac\\
CEA, List\\
Université Paris-Saclay\\
F-91120 Palaiseau, France
\And
C\'edric Gouy-Pailler\\
CEA, List\\
Université Paris-Saclay\\
F-91120 Palaiseau, France
}
\else
\usepackage{authblk}

\newbox{\orcid}\sbox{\orcid}{\includegraphics[scale=0.06]{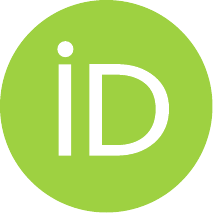}} 
\author[1]{%
	\href{https://orcid.org/0000-0000-0000-0000}{\usebox{\orcid}\hspace{1mm}David S.~Hippocampus\thanks{\texttt{hippo@cs.cranberry-lemon.edu}}}%
}
\author[1,2]{%
	\href{https://orcid.org/0000-0000-0000-0000}{\usebox{\orcid}\hspace{1mm}Elias D.~Striatum\thanks{\texttt{stariate@ee.mount-sheikh.edu}}}%
}
\affil[1]{Department of Computer Science, Cranberry-Lemon University, Pittsburgh, PA 15213}
\affil[2]{Department of Electrical Engineering, Mount-Sheikh University, Santa Narimana, Levand}
\fi

\renewcommand{\headeright}{}
\renewcommand{\undertitle}{Technical Report}
\renewcommand{\shorttitle}{Dataset Dictionary Learning in a Wasserstein Space for Federated Domain Adaptation}

\hypersetup{
pdftitle={Dataset Dictionary Learning in a Wasserstein Space for Federated Domain Adaptation},
pdfsubject={cs.LG},
pdfauthor={Montesuma,Espinoza,Mboula,Mayoue,Souloumiac,Gouy-Pailler},
pdfkeywords={Federated Learning,Domain Adaptation,Dataset Dictionary Learning,Optimal Transport},
}

\begin{document}
\maketitle

\maketitle

\input{content/0_abstract}
\input{content/1_intro}

\input{content/2_related_work}
\input{content/3_methodology}
\input{content/4_experiments}
\input{content/5_conclusion}

\bibliographystyle{unsrt}
\bibliography{references}  

\newpage

\appendix

\section{Introduction}

This appendix is organized as follows. Section 2 provides a proof for Theorem 3.1, and Section 3 provides further details on how we experiment with third-party code.

\section{Proof of Theorem 3.1}

\begin{theorem}
    Let $(\mathcal{P},\mathcal{A})$ be a dictionary, and $\epsilon \in \mathbb{R}^{d}$ be a random perturbation. Let $\tilde{\mathcal{P}} = \{\tilde{P}_{k}\}_{k=1}^{K}$ s.t.,
    \begin{equation*}
        \tilde{P}_{k}(\mathbf{z},\mathbf{y}) = \dfrac{1}{n}\sum_{i=1}^{n}\delta((\mathbf{z},\mathbf{y}) - (\mathbf{z}_{j}^{(P_{k})}+\epsilon, \mathbf{y}_{j}^{(P_{k})})),
    \end{equation*}
    then,
    \begin{align*}
        f(\tilde{\mathcal{P}},\mathcal{A}) &= f(\mathcal{P},\mathcal{A}) + 2\epsilon^{T}\nabla_{x}f + \lVert \epsilon\rVert_{2}^{2},
    \end{align*}
\end{theorem}

\begin{proof}
    Our proof relies on the following observation,
    \begin{align*}
        \tilde{\mathbf{x}}_{i}^{(B_{\ell})} &= n\sum_{k=1}^{K}\alpha_{k}\sum_{j=1}^{n}\pi_{ij}^{(k)}\tilde{\mathbf{x}}_{i}^{(P_{k})},\\
        &= n\sum_{k=1}^{K}\alpha_{k}\sum_{j=1}^{n}\pi_{ij}^{(k)}(\mathbf{x}_{i}^{(P_{k})} + \epsilon),\\
        &= \mathbf{x}_{i}^{(B_{\ell})} + \epsilon\underbrace{(n\sum_{k=1}^{K}\alpha_{k}\sum_{j=1}^{n}\pi_{ij}^{(k)})}_{=1},\\
        &= \mathbf{x}_{i}^{(B_{\ell})} + \epsilon.
    \end{align*}
    so that,
    \begin{align*}
        f(\tilde{\mathcal{P}}, \mathcal{A}) &= \sum_{i=1}^{n}\sum_{i'=1}^{n}\pi_{i,i'}\lVert \tilde{\mathbf{x}}_{i}^{(B_{\ell})} - \mathbf{x}_{i'}^{(Q_{\ell})} \rVert_{2}^{2},\\
        &= \sum_{i=1}^{n}\sum_{i'=1}^{n}\pi_{i,i'}\lVert \mathbf{x}_{i}^{(B_{\ell})} + \epsilon - \mathbf{x}_{i'}^{(Q_{\ell})} \rVert_{2}^{2},
    \end{align*}
    here, note that, $\lVert (\mathbf{x}_{i}^{(B_{\ell})} + \epsilon) - \mathbf{x}_{i'}^{(Q_{\ell})} \rVert_{2}^{2} = \lVert (\mathbf{x}_{i}^{(B_{\ell})} - \mathbf{x}_{i'}^{(Q_{\ell})}) + \epsilon\rVert_{2}^{2}$, which leads to,
    \begin{equation*}
        \lVert \mathbf{x}_{i}^{(B_{\ell})} - \mathbf{x}_{i'}^{(Q_{\ell})} \rVert_{2}^{2} + 2\epsilon^{T}(\mathbf{x}_{i}^{(B_{\ell})} - \mathbf{x}_{i'}^{(Q_{\ell})}) + \lVert \epsilon \rVert_{2}^{2}
    \end{equation*}
    then,
    \begin{align*}
        f(\tilde{\mathcal{P}}, \mathcal{A}) = \sum_{i=1}^{n}\sum_{i'=1}^{n}\pi_{i,i'}\biggr{(}&\lVert \mathbf{x}_{i}^{(B_{\ell})} - \mathbf{x}_{i'}^{(Q_{\ell})} \rVert_{2}^{2} + 2\epsilon^{T}(\mathbf{x}_{i}^{(B_{\ell})} - \mathbf{x}_{i'}^{(Q_{\ell})}) + \lVert \epsilon \rVert_{2}^{2}\biggr{)}.
    \end{align*}
    Breaking the summation into three terms,
    \begin{align*}
        f(\tilde{\mathcal{P}}, \mathcal{A}) = f(\mathcal{P},\mathcal{A}) &+ \epsilon^{T}\underbrace{\biggr(2\sum_{i=1}^{n}\sum_{i'=1}^{n}\pi_{i,i'}(\mathbf{x}_{i}^{(B_{\ell})} - \mathbf{x}_{i'}^{(Q_{\ell})})\biggr)}_{=\nabla_{x}f} + \lVert \epsilon \rVert_{2}^{2}.
    \end{align*}
\end{proof}

\section{Experiments}

In this section, we give details about the hyper-parameters used in our paper, for the sake of reproducibility. These concern third-party code open-sourced on Github or OpenReview from the methods used in our paper.

\subsection{Reproducing the State-of-the-art}

In our experiments, we execute the code of 3 state-of-the-art methods, namely, FADA\footnote{Code available at \url{https://openreview.net/forum?id=HJezF3VYPB}}~\cite{peng2019federated}, KD3A\footnote{Code available at \url{https://github.com/FengHZ/KD3A/tree/master}}~\cite{feng2021kd3a} and Co-MDA\footnote{Code available at \url{https://github.com/Xinhui-99/CoMDA}}~\cite{liu2023co}. We selected these methods due to their relevance, and the availability of open source code. These resources allowed us to make fair comparisons to the existing methods. Here, we make some remarks about how we run third party code,
\begin{enumerate}
    \item FADA~\cite{peng2019federated} runs the k-Means clustering algorithm within its fit procedure, where the number of clusters equals the number of classes $n_{c}$. Due to the internal workings of the authors code, the input for k-Means must have \emph{at least $n_{c}$ elements}. For datasets with small number of classes, such as Caltech-Office, ImageClef and Office 31, we were able to run the authors code. Howerver, the batch size required for running their code on Office-Home ($65$ classes) and Adaptiope ($123$ classes) was beyond the hardware capabilities used in this work.
    \item All of our experiments, including those with KD3A, Co-MDA and FADA, use \emph{torchvision} ResNets as backbones with pre-trained weights on ImageNet~\cite{deng2009imagenet}. We use the \emph{IMAGENET1K\_V2} weights.
\end{enumerate}

With our design choices, we were able to improve the average adaptation performance of KD3A on Office-Home to $76.2\%$, in comparison with what was previously reported by~\cite{liu2023co}.

\subsection{Hyper-Parameter Settings}

\noindent\textbf{FedAVG} has hyper-parameter associated with its training process, i.e., batch size, number of epochs, learning rate and weight decay. Globally, we use an \gls{sgd} optimizer with a momentum term of $0.9$. On all datasets we use a mini-batches of $32$ samples. The training is conducted for $12$ epochs, where an epoch corresponds to a complete run through the entire dataset of all clients. The learning rate is $10^{-2}$ and we use a weight decay term of 5$\times 10^{-4}$. Like previous works on decentralized MSDA~\cite{feng2021kd3a,liu2023co}, we average clients weights at the end of each epoch.

\noindent\textbf{KD3A and Co-MDA.} On their respective repositories, the authors of KD3A and Co-MDA present the hyper-parameters of their methods. For Adaptiope, we re-use the hyper-parameters used for DomainNet. 

\begin{table}[ht]
    \centering
    \caption{Hyper-parameter setting for KD3A and Co-MDA.}
    \begin{tabular}{lcccc}
        \toprule
        Benchmark & {Batch Size} & {\# Epochs} & {Confidence Gate} & {Learning Rate} \\
        \midrule
        ImageCLEF & 32 & 100 & $\{0.9, 0.95\}$ & $10^{-3}$\\
        Office 31 & 32 & 100 & $\{0.9,0.95\}$ & $10^{-2}$\\
        Office Home & 32 & 100 & $\{0.9, 0.95\}$ & $10^{-2}$\\
        Adaptiope & 32 & 100 & $\{0.8,0.95\}$ & $10^{-2}$\\
        \bottomrule
    \end{tabular}
    \label{tab:my_label}
\end{table}

\noindent\textbf{FedaDiL.} For our method, we performed grid-search on the batch size $n_{b}$, number of atoms $K$ and number of samples $n$. For the number of atoms, we search over $K \in \{2, \cdots, 6\}$. We further parametrize $n_{b}$ and $n$ by the number of classes $n_{c}$ (c.f., Table 1 in our main paper). We display the best parameters in table 2 below, alongside the cost of communication for this set of parameters. Note that, as we explore in our experiments, the performance of our method is robust with respect the choice of hyper-parameters.

\begin{table}[ht]
    \centering
    \caption{Hyper-parameter setting of \gls{dadil} alongside relative communication cost (in \%) in comparison with transmitting the parameters of a ResNet network. In all cases, DaDiL is more efficient than communicating the parameters of a neural net. $\downarrow$ denotes that lower is better.}
    \resizebox{\linewidth}{!}{\begin{tabular}{llccc|cccc}
        \toprule
        Benchmark & Backbone & Batch Size & \# Atoms & \# Samples & Batch Size & \# Atoms & \# Samples & \\
        \midrule
        & & \multicolumn{3}{c|}{DaDiL-R} & \multicolumn{3}{c}{DaDiL-E}\\
        ImageCLEF & ResNet50 & 240 & 6 & 840 & 180 & 5 & 1440\\
        Caltech-Office 10 & ResNet101 & 100 & 3 & 500 & 50 & 3 & 500\\
        Office31 & ResNet50 & 465 & 3 & 2170 & 465 & 3 & 1550\\
        Office Home & ResNet101 & 520 & 3 & 1950 & 325 & 2 & 1950 \\
        Adaptiope & ResNet101 & 615 & 4 & 3690 & 615 & 4 & 3690   \\
        \bottomrule
    \end{tabular}}
    \label{tab:hyper-param-setting}
\end{table}

\end{document}

%% file: content/preamble.tex
\usepackage{graphicx}
\usepackage{amsmath, amsthm}
\usepackage{amssymb}
\usepackage{booktabs}
\usepackage{fontawesome}
\usepackage{color}
\usepackage{tikzit}
\usepackage{algorithm}
\usepackage{algpseudocode}
\usepackage{nicefrac}
\usepackage{colortbl}
\usepackage{wrapfig}
\usepackage{subcaption}
\usepackage[acronym]{glossaries}

\newacronym{sgd}{SGD}{Stochastic Gradient Descent}
\newacronym{da}{DA}{Domain Adaptation}
\newacronym{spc}{SPC}{Samples per Class}
\newacronym{erm}{ERM}{Empirical Risk Minimization}
\newacronym{otda}{OTDA}{Optimal Transport for Domain Adaptation}
\newacronym{msda}{MSDA}{Multi-Source DA}
\newacronym{fl}{FL}{Federated Learning}
\newacronym{fda}{FDA}{Federated DA}
\newacronym{ot}{OT}{Optimal Transport}
\newacronym{mtl}{MTL}{Multitask Learning}
\newacronym{dil}{DiL}{Dictionary Learning}
\newacronym{dadil}{DaDiL}{Dataset Dictionary Learning}
\newacronym{feddadil}{FedDaDiL}{Federated Dataset Dictionary Learning}
\newacronym{copa}{COPA}{Collaborative Optimization and Aggregation}
\newacronym{fada}{FADA}{Federated Adversarial Domain Adaptation}
\newacronym{kd3a}{KD3A}{Knowledge Distillation and Decentralized Domain Adaptation}
\newacronym{comda}{Co-MDA}{Co$^{2}$-Learning with Multi-Domain Attention}

\renewcommand{\inf}[1]{\underset{#1}{\text{inf}}\,}

\newcommand{\myiid}[0]{\overset{\text{i.i.d.}}{\sim}}

\newtheorem{theorem}{Theorem}[section]
\newtheorem{definition}{Definition}[section]

%% file: content/0_abstract.tex
\begin{abstract}
Multi-Source Domain Adaptation (MSDA) is a challenging scenario where multiple related and heterogeneous source datasets must be adapted to an unlabeled target dataset. Conventional MSDA methods often overlook that data holders may have privacy concerns, hindering direct data sharing. In response, decentralized MSDA has emerged as a promising strategy to achieve adaptation without centralizing clients' data. Our work proposes a novel approach, Decentralized Dataset Dictionary Learning, to address this challenge. Our method leverages Wasserstein barycenters to model the distributional shift across multiple clients, enabling effective adaptation while preserving data privacy. Specifically, our algorithm expresses each client's underlying distribution as a Wasserstein barycenter of public atoms, weighted by private barycentric coordinates. Our approach ensures that the barycentric coordinates remain undisclosed throughout the adaptation process. Extensive experimentation across five visual domain adaptation benchmarks demonstrates the superiority of our strategy over existing decentralized MSDA techniques. Moreover, our method exhibits enhanced robustness to client parallelism while maintaining relative resilience compared to conventional decentralized MSDA methodologies.
\end{abstract}

\keywords{Federated Learning \and Domain Adaptation \and Dataset Dictionary Learning \and Optimal Transport}

%% file: content/1_intro.tex
\section{Introduction}

Supervised machine learning models are trained with large amounts of labeled data. However, these models are subject to performance degradation, if the data used for training does not exactly resembles those used for test. This issue is known in the literature as dataset, or distributional shift~\cite{quinonero2008dataset}. For instance, in computer vision, factors such as illumination, pose and image quality can induce changes in the data underlying distribution~\cite{saenko2010adapting,wang2018deep}. In this context \gls{msda}~\cite{crammer2008learning,montesuma2023learning} emerged as a strategy to adapt multiple, heterogeneous, labeled source datasets towards an unlabeled target dataset.

Nonetheless, standard methods in \gls{msda} overlook that datasets may be divided over multiple clients, rather than centralized on a server. Due to privacy concerns, these clients may not want to centralize their data. Motivated by this challenge, decentralized \gls{msda} is a possible solution to this problem~\cite{peng2019federated,feng2021kd3a,liu2023co}. In parallel to this strategy, distributional shift has been considered by the federated learning literature~\cite{mcmahan2017communication,bai2024benchmarking}. However, concerning domain adaptation, \textbf{non-i.i.d. federated learning strategies are limited since they do not exploit unlabeled target domain data}.

Existing methods in decentralized \gls{msda} mainly follow 2 strategies. First, one may align the multiple existing distributions in a latent space, by learning invariant features~\cite{peng2019federated}. Second, on top of aligning distributions, one may perform pseudo-labeling of the target domain through classifiers learned on the source domains~\cite{feng2021kd3a,liu2023co}. While invariant representation learning is an important component in domain adaptation, it poses a trade-off between classification performance and domain invariance~\cite{zhao2022fundamental}, which may limit domain adaptation performance. \textbf{We thus take a different route by modeling the shift between domains in a Wasserstein space}.

In parallel, \gls{ot} is a mathematical theory concerned with the displacement of mass at least effort. This theory previously contributed to the diverse landscape of \gls{da}, both in single-source~\cite{courty2017otda,courty2017joint} and multi-source~\cite{montesuma2021icassp,montesuma2021cvpr,turrisi2022multi} settings. \gls{ot} is especially advantageous, as it is sensitive to the geometry of the data ambient space~\cite{montesuma2023recent}. In addition, it allows for the definition of \emph{averages of probability distributions} through Wasserstein barycenters~\cite{agueh2011barycenters}, which has been previously used for \gls{da}~\cite{montesuma2021icassp,montesuma2021cvpr,montesuma2023learning}. Overall, \gls{ot} presents a principled framework for developing \gls{da} algorithms~\cite{redko2017theoretical}.

Given the aforementioned limitations of invariant representation learning, we offer a novel, \gls{ot}-based method that learns how to express clients' underlying probability distribution as a Wasserstein barycenter of learned atoms, weighted by barycentric coordinates. As such, our method effectively keeps clients' distributions private, because their barycentric coordinates do not need to be communicated. To the best of our knowledge, ours is the first \gls{ot}-inspired decentralized \gls{da} algorithm, which \textbf{does not align clients' distributions}. Our contributions are threefold,
\begin{enumerate}
    \item We propose a novel strategy, called \gls{feddadil}, for performing decentralized dictionary learning over empirical distributions. This strategy has the advantage of keeping the variables that allow to reconstruct clients' data distributions, i.e., the barycentric coordinates, private.
    \item We provide a novel theoretical analysis of the objective function of \gls{dadil}~\cite{montesuma2023learning}, showing that, for small perturbations, it behaves as a quadratic form on the atoms' feature vectors.
    \item We provide extensive empirical results on five visual \gls{da} benchmarks, showing that (i) our strategy has superior performance on all datasets, (ii) \gls{feddadil} is lightweight compared to communicating the parameters of deep neural nets. Over the tested benchmarks, communicating dictionaries with the server corresponds to approximately $34.6\% \pm 23.3\%$ of the total amount of bits used to encode ResNet weights, and (iii) \gls{feddadil} is more robust w.r.t. client parallelism than previous methods.
\end{enumerate}
Especially, this paper extends the previous work~\cite{castellon2024federated} in two important ways. First, we average atoms in a similar way to FedAVG, which proved to be effective (c.f., section~\ref{sec:experiments}). Second, we provide a new theoretical analysis on the loss function being minimizing throughout dictionary learning.

The rest of this paper is organized as follows. Section~\ref{sec:related_work} presents related work on dictionary learning and decentralized \gls{msda}. Section~\ref{sec:methodology} presents our novel approach for decentralized \gls{msda}, constituted of \emph{FedAVG} (section~\ref{sec:learning_encoder}) and \emph{\gls{feddadil}} (section~\ref{sec:fed_dadil}). Section~\ref{sec:experiments} presents our experiments on various visual \gls{da} benchmarks. Finally, section~\ref{sec:conclusion} concludes this paper.

%% file: content/2_related_work.tex
\section{Related Work}\label{sec:related_work}

\noindent\textbf{Dictionary Learning} seeks to decompose a set of vectors $\{\mathbf{x}_{1},\cdots,\mathbf{x}_{N}\}$, $\mathbf{x}_{\ell} \in \mathbb{R}^{d}$, as a linear combination of atoms $\{\mathbf{p}_{1},\cdots\mathbf{p}_{K}\}$, weighted by representation vectors $\{\alpha_{1},\cdots,\alpha_{N}\}$. When $\mathbf{x}_{\ell}$ are histograms (i.e., $\sum_{j}x_{\ell,j} =1 $ and $x_{\ell,j} \geq 0$), \gls{ot} defines meaningful loss functions~\cite{rolet2016fast}, as well as novel ways of aggregating atoms~\cite{schmitz2018wasserstein}. Additionally, decentralized strategies for learning dictionaries have been proposed by~\cite{gkillas2022federated} and \cite{liang2023personalized}, respectively. In this work we use a different interpretation of dictionary learning, in which the atoms are empirical probability distributions~\cite{montesuma2023learning}. To the best of our knowledge, ours is the first decentralized algorithm for dictionaries of empirical distributions.

\noindent\textbf{Decentralized Domain Adaptation.} Dataset heterogeneity is a major challenge in supervised learning, as commonly used tools, such as \gls{erm}~\cite{vapnik1991principles} work under the assumption of i.i.d. data. This issue also emerges in decentralized settings, such as federated learning~\cite{mcmahan2017communication}, in which client data follow different probability distributions~\cite{gao2022survey}. In this paper we consider decentralized \gls{msda}, where on top of clients with different probability distributions, we have an unique client, called \emph{target}, who do not possess labeled data. Different ideas in \gls{msda} have been adapted to the decentralized setting. \gls{fada}~\cite{peng2019federated} uses the adversarial learning framework of~\cite{ganin2016domain} with an additional feature disentanglement~\cite{bengio2013representation} module. \gls{kd3a}~\cite{feng2021kd3a} uses knowledge distillation~\cite{hinton2015distilling,meng2018adversarial} and pseudo-labeling for the adaptation. Finally, \gls{comda}~\cite{liu2023co} studies black box \gls{da}, a type of source-free \gls{da}~\cite{yu2023comprehensive}, where one has access only to outputs of source client models.



%% file: content/3_methodology.tex
\section{Proposed Approach}\label{sec:methodology}

\noindent\textbf{Problem Statement.} We are interested in decentralized \gls{msda}. Clients are associated with indices $\ell=1,\cdots,N$, where $\ell = 1,\cdots,N-1$ are called \emph{source clients} and $\ell = N$ is the \emph{target} client. In what follows, we have access to pairs $\{(\mathbf{x}_{i}^{(Q_{\ell})}, \mathbf{y}_{i}^{(Q_{\ell})})\}_{i=1}^{n_{\ell}}$ from the sources, where $\mathbf{x}_{i}^{(Q_{\ell})} \in \mathbb{R}^{d}$ and $\mathbf{y}_{i}^{(Q_{\ell})} \in \Delta_{K}$, i.e., $\sum_{c}y_{ic}=1$ and $y_{ic} \geq 0$. For $\ell = N$, we only have access to $\{\mathbf{x}_{i}^{(Q_{\ell})}\}_{i=1}^{n_{\ell}}$, i.e., samples are not labeled. Our goal is to learn a classifier on $Q_{N}$, based on the available samples. We do so through \textbf{dataset dictionary learning}, i.e., we learn a set $\mathcal{P} = \{\hat{P}_{k}\}_{k=1}^{K}$ and $\mathcal{A} = \{\alpha_{\ell}\}_{\ell=1}^{N}$, $\alpha_{\ell} \in \Delta_{K}$ such that $\hat{Q}_{N} = \mathcal{B}(\alpha_{N};\mathcal{P})$, i.e., the barycenter in Wasserstein space of $\mathcal{P}$. We detail these ideas in the following.

\subsection{Background}\label{sec:background}

We start with \gls{ot}, a field of mathematics that, in a nutshell, studies the transportation of probability distributions under least effort. Our work is based on the Kantorovich formulation~\cite{kantorovich1942transfer}. For recent overviews of the theory, we refer readers to~\cite{peyre2019computational} and~\cite{montesuma2023recent}. In the following, we approximate client distributions empirically through,
\begin{align}
    \hat{Q}_{\ell}(\mathbf{x}) = \dfrac{1}{n_{\ell}}\sum_{i=1}^{n_{\ell}}\delta(\mathbf{x}-\mathbf{x}_{i}^{(Q_{\ell})}).\label{eq:empirical_approx}
\end{align}
where $\delta$ is the Dirac delta function. In this context, $\mathbf{X}^{(Q_{\ell})} \in \mathbb{R}^{n_{\ell}\times d}$ is called \emph{the support of} $\hat{Q}_{\ell}$. In this setting, the so-called \gls{ot} problem between distributions $\hat{P}$ and $\hat{Q}$ is,
\begin{equation}
\hat{\gamma} = \argmin{\gamma\in \Gamma(P,Q)} \sum_{i=1}^{n}\sum_{j=1}^{m}\gamma_{ij}C_{ij},\label{eq:KantorovichProblem}
\end{equation}
where $C_{ij} = c(\mathbf{x}_{i}^{(P)},\mathbf{x}_{j}^{(Q)})$ is called \emph{ground-cost matrix} and $c$ is a function modeling the \emph{effort of transportation} between samples $\mathbf{x}_{i}^{(P)} \myiid P$ and $\mathbf{x}_{j}^{(Q)} \myiid P$. The matrix $\gamma \in \mathbb{R}^{n\times m}$ is called \emph{transport plan}, and $\Gamma(P,Q)=\{\gamma:\sum_{i}\gamma_{ij}=m^{-1}\text{ and }\sum_{j}\gamma_{ij}=n^{-1}\}$.

Problem~\ref{eq:KantorovichProblem} is a linear program over the variables $\gamma_{ij}$. As such, it has $\mathcal{O}(n^{3}\log n)$ complexity over the number of samples~\cite{dantzig1983reminiscences}. A way of alleviating this complexity is using mini-batch \gls{ot}~\cite{fatras2020learning}, instead of calculating $\hat{\gamma}$ over all samples of $\hat{P}$ and $\hat{Q}$. In addition, $\gamma$ can be used for building a mapping between distributions $\hat{P}$ and $\hat{Q}$. In the discrete case, this takes the form of the \emph{barycentric projection}~\cite{courty2017otda},
\begin{align}
    T_{\gamma}(\mathbf{x}_{i}^{(P)}) = \argmin{\mathbf{x} \in \mathbb{R}^{d}}\sum_{j=1}^{m}\gamma_{ij}c(\mathbf{x},\mathbf{x}_{j}^{(Q)}),\label{eq:barymap}
\end{align}
when $c(\mathbf{x}_{i}^{(P)},\mathbf{x}_{j}^{(Q)}) = \lVert \mathbf{x}_{i}^{(P)} -\mathbf{x}_{j}^{(Q)} \rVert_{2}^{2}$ eq.~\ref{eq:barymap} can be conveniently expressed as $T_{\gamma}(\mathbf{X}^{(P)}) = n\gamma\mathbf{X}^{(Q)}$.

Based on the \gls{ot} solution, one has an associated \emph{cost} for transporting $\hat{P}$ to $\hat{Q}$, defined as,
\begin{align}
    \mathcal{T}_{c}(\hat{P},\hat{Q}) = \sum_{i=1}^{n}\sum_{j=1}^{m}\gamma_{ij}^{\star}C_{ij},\label{eq:transport_effort}
\end{align}
where $\gamma^{\star}$ is the solution of eq.~\ref{eq:KantorovichProblem}. Henceforth, we use $\mathcal{T}_{2}$ to eq.~\ref{eq:transport_effort} with $c = \lVert \cdot \rVert_{2}^{2}$. When $C_{ij} = \lVert \mathbf{x}_{i}^{(P)} - \mathbf{x}_{j}^{(Q)} \rVert^{p}_{2}$, one may define $W_{p}(\hat{P},\hat{Q}) = (\mathcal{T}_{c}(\hat{P},\hat{Q}))^{1/p}$, the Wasserstein distance between $\hat{P}$ and $\hat{Q}$, which \emph{inherits the metric properties from }$C$. This metric between distributions allows for the definition of \textbf{barycenters} of probability distributions~\cite{agueh2011barycenters},

\begin{definition}
For a set of distributions $\mathcal{P} = \{P_{k}\}_{k=1}^{K}$ and barycentric coordinates $\alpha \in \Delta_{K}$, the Wasserstein barycenter is a solution to,
\begin{align}
    B^{\star} = \mathcal{B}(\alpha;\mathcal{P}) = \inf{B}\sum_{k=1}^{K}\alpha_{k}W_{2}(P_{k}, B)^{2}.\label{eq:true_bary}
\end{align}
Henceforth we call $\mathcal{B}(\cdot;\mathcal{P})$ barycentric operator.
\end{definition}


Even though eq.~\ref{eq:true_bary} is continuous, the barycenter problem can be solved in terms of the support $\mathbf{X}^{(B)}$ of $B$, as described in~\cite{cuturi2013sinkhorn}. Next, we extend the \gls{ot} problem for handling labeled data. This is done by integrating the labels into the cost function~\cite{montesuma2021cvpr}. As~\cite{montesuma2023learning}, we use,
\begin{align}
    C_{ij} = \lVert \mathbf{x}_{i}^{(P)} - \mathbf{x}_{j}^{(Q)} \rVert_{2}^{2} + \beta \lVert \mathbf{y}^{(P)}_{i}-\mathbf{y}^{(Q)}_{j}\rVert_{2}^{2},\label{eq:supervised_ground_cost}
\end{align}
where $\beta > 0$ controls the importance of penalizing the transport between samples from different classes. Henceforth, we use $\mathcal{T}_{c}$ to denote eq.~\ref{eq:transport_effort} with $C_{ij}$ given by eq.~\ref{eq:supervised_ground_cost}. Furthermore, one may solve eq.~\ref{eq:true_bary} for $\hat{P}_{k}$ with support $(\mathbf{X}^{(P_{k})}, \mathbf{Y}^{(P_{k})})$, which yields $\hat{B}$ with support $(\mathbf{X}^{(B)},\mathbf{Y}^{(B)})$ where $\mathbf{Y}^{(B)} = \{\mathbf{y}_{i}^{(B)}\}_{i=1}^{n}$ are soft-labels, i.e., $\mathbf{y}_{i}^{(B)} \in \Delta_{n_{c}}$. This computation is done, for instance, with~\cite[Alg. 1]{montesuma2023learning}.

In~\cite{montesuma2023learning}, authors presented a novel dictionary learning framework over \emph{empirical distributions} (cf. eq.~\ref{eq:empirical_approx}). The \gls{dadil} framework introduces \emph{virtual distributions}, $\mathcal{P} = \{\hat{P}_{k}\}_{k=1}^{K}$, called atoms. Each $\hat{P}_{k}$ has a free-support, $(\mathbf{X}^{(P_{k})}, \mathbf{Y}^{(P_{k})})$, which is determined via optimization. The atoms are linked to true distributions $\hat{Q}_{\ell} \in \mathcal{Q}$ through Wasserstein barycenter $\mathcal{B}(\alpha_{\ell};\mathcal{P})$, where $\alpha_{\ell}$ are the barycentric coordinates allowing to \emph{reconstruct} $\hat{Q}_{\ell}$. Mathematically, \gls{dadil} is expressed as,
\begin{align}
    (\mathcal{P}^{\star},\mathcal{A}^{\star}) = \argmin{\mathcal{P},\mathcal{A}\in(\Delta_{K})^{N}}\dfrac{1}{N}\sum_{\ell=1}^{N}\mathcal{L}(\hat{Q}_{\ell}, \mathcal{B}(\alpha_{\ell};\mathcal{P})),\label{eq:dadil}
\end{align}
where $\mathcal{L} = \mathcal{T}_{c}$ if $\hat{Q}_{\ell}$ is labeled (i.e., $\ell \leq N-1$) or $\mathcal{L}=\mathcal{T}_{2}$ otherwise (i.e., $\ell = N$). In a nutshell, \gls{dadil} learns to approximate true distributions $\hat{Q}_{\ell}$ as the Wasserstein barycenter of atoms $\mathcal{P}$. In practice, \gls{dadil} relies on features extracted by a neural net, so that barycenters are calculated in a semantically rich latent space. Furthermore,~\cite{montesuma2023learning} minimizes eq.~\ref{eq:dadil} via mini-batches.

Our method parts from the following observation: \textbf{while the atoms $\mathcal{P}$ are shared by all domains, the barycentric coordinates $\alpha_{\ell}$ are specific to each domain, and thus do are not aggregated nor communicated throughout federated dictionary learning}. Next, we describe the two ingredients for our \gls{feddadil} strategy.

\subsection{Federated Learning an Encoder Network}\label{sec:learning_encoder}

\begin{figure*}[ht]
    \centering
    \begin{subfigure}{0.45\linewidth}
        \includegraphics[width=\linewidth]{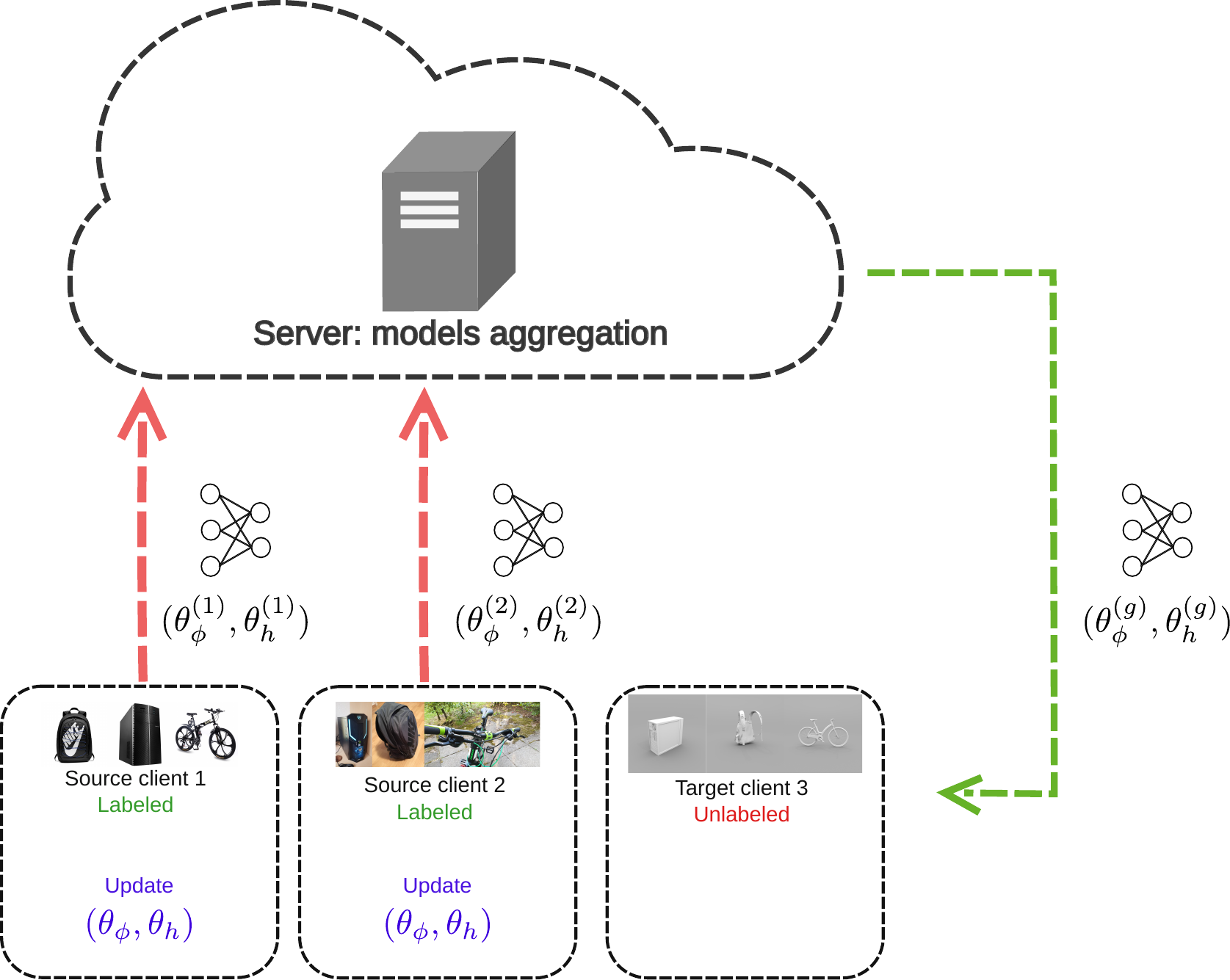}
        \caption{Step 1: FedAVG}
    \end{subfigure}\hfill
    \begin{subfigure}{0.45\linewidth}
        \includegraphics[width=\linewidth]{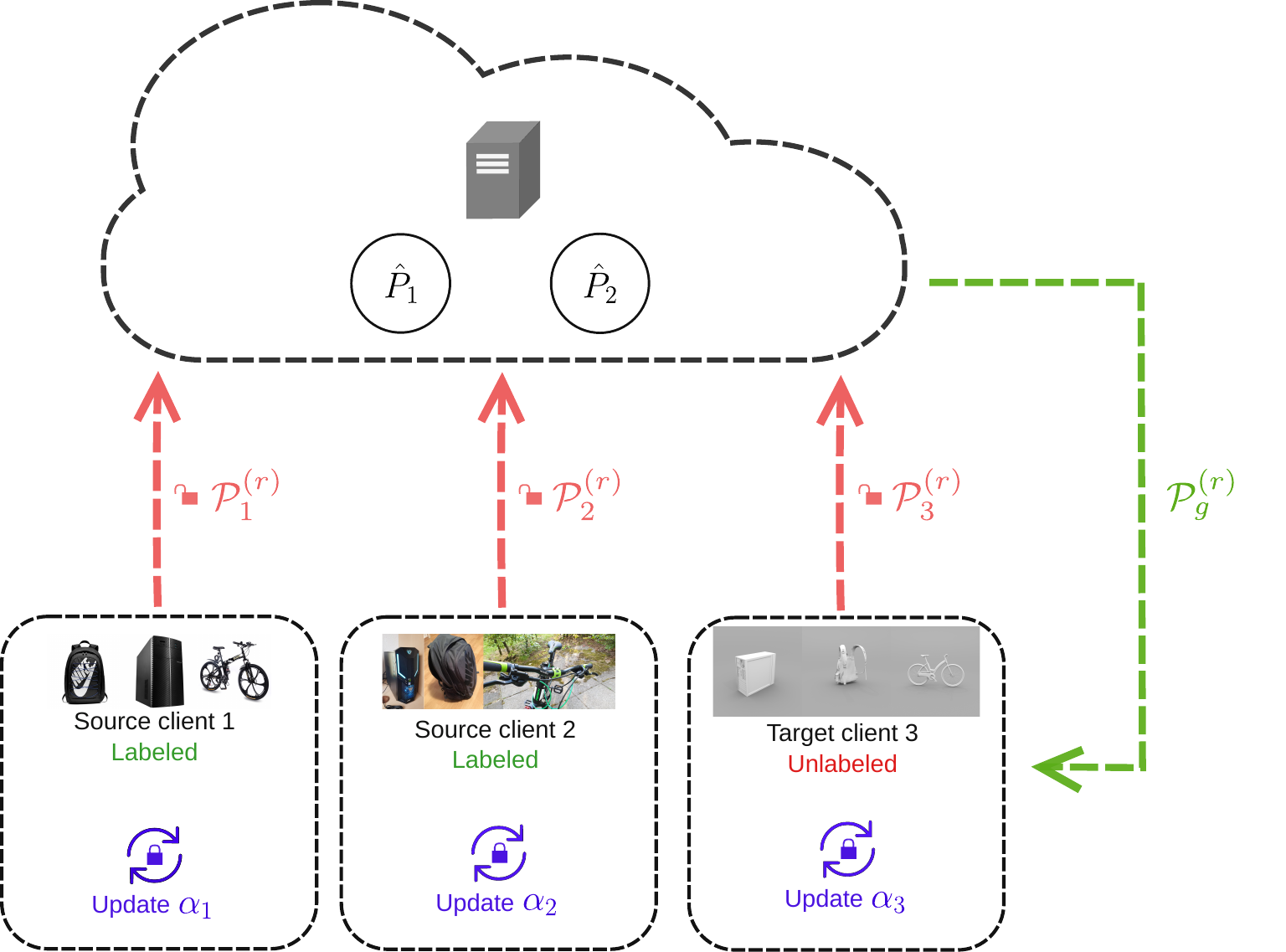}
        \caption{Step 2: \gls{feddadil}}
    \end{subfigure}
    \caption{Illustration of our decentralized \gls{msda} strategy. (a) We fit a neural deep neural network composed of an encoder net $\phi$ and a classifier $h$, without centralizing client data. In principle, the target client does not participate at this step, unless some adaptation method is used (e.g., KD3A~\cite{feng2021kd3a}). (b) We do the adaptation step with features extracted from the fine-tuned source model, through our proposed \gls{feddadil}.}
    \label{fig:feddadil}
\end{figure*}

Our \gls{feddadil} framework works over extracted features of a deep neural net. In a decentralized setting, clients cannot centralize their data for fine-tuning an existing architecture. To keep our overall pipeline \emph{end-to-end decentralized}, we choose to fine-tune an encoder using the FedAVG algorithm of~\cite{mcmahan2017communication}. Contrary to existing works~\cite{peng2019federated,feng2021kd3a}, \textbf{we do not align the probability distributions of clients' features}. This choice is important, because the dictionary learning step must have a rich variety of probability distributions for modeling distributional shift.

In this context, let $\phi:\mathcal{X}\rightarrow\mathcal{Z}$ be the encoder network, which takes as inputs images $\mathbf{x}_{i}^{(Q_{\ell})}$ and outputs a vector $\mathbf{z}_{i}^{(Q_{\ell})}=\phi(\mathbf{x}_{i}^{(Q_{\ell})}) \in \mathbb{R}^{d}$. Likewise, let $h:\mathcal{Z}\rightarrow\mathcal{Y}$ be a classifier. For instance, $\phi$ is the set of convolutional layers in a ResNet~\cite{he2016deep}, whereas $h$ is a single-layer Perceptron. We find the parameters $\theta_{\phi}$ and $\theta_{h}$ by minimizing,
\begin{equation}
(\theta_{\phi}^{\star},\theta_{h}^{\star}) =\argmin{\theta_{\phi}, \theta_{h}}\dfrac{1}{N-1}\sum_{\ell=1}^{N-1}\dfrac{1}{n_{\ell}}\sum_{i=1}^{n_{\ell}}\mathcal{L}(\mathbf{x}_{i}^{(Q_{\ell})},\mathbf{y}_{i}^{(Q_{\ell})};\theta_{\phi},\theta_{h}),\label{eq:fed_avg}
\end{equation}
where $\mathcal{L}(\mathbf{x}_{i}^{(Q_{\ell})},\mathbf{y}_{i}^{(Q_{\ell})};\theta_{\phi},\theta_{h}) = \sum_{c=1}^{n_{c}}y_{ic}\log h(\phi(\mathbf{x}_{i}))_{c}$ is the cross-entropy loss. Eq.~\ref{eq:fed_avg} is minimized in a federated way, i.e., each client $\ell \leq N-1$ minimizes the loss with respect its own data independently of others. We explicitly leave the target client out of the training process, since they do not have labeled data.

As discussed in~\cite{mcmahan2017communication}, the federated setting introduces a new level of complexity to the learning problem. As a consequence, before training begins, \textbf{a server} needs to initialize the parameters $\{\theta_{\phi}^{(\ell)},\theta_{h}^{(\ell)}\}_{\ell=1}^{N-1}$ with the same values. After a fixed number of $E$ training steps of each client, the server \textbf{aggregates clients weights}, by averaging clients' versions,
\begin{align}
    \theta^{(g)}_{\phi}=\dfrac{1}{N-1}\sum_{\ell=1}^{N-1}\theta_{\phi}^{(\ell)}\text{, and, }\theta^{(g)}_{h}=\dfrac{1}{N-1}\sum_{\ell=1}^{N-1}\theta_{h}^{(\ell)}.\label{eq:server_aggregation}
\end{align}
 
\subsection{Federated Dataset Dictionary Learning}\label{sec:fed_dadil}

We assume $\theta_{\phi}^{(g)}$ fixed, and $\mathbf{z}_{i}^{(Q_{\ell})} = \phi(\mathbf{x}_{i}^{(Q_{\ell})};\theta_{\phi}^{(g)})$. In \gls{feddadil}, each atom $\hat{P}_{k}$ has a free support, denoted as $(\mathbf{Z}^{(P_{k})}, \mathbf{Y}^{(P_{k})})$ and each client holds a set of barycentric coordinates $\alpha_{\ell} \in \Delta_{k}$. Hence,
\begin{align}
(\mathcal{P}^{\star},\mathcal{A}^{\star}) = \argmin{\mathcal{P},\mathcal{A}}\underbrace{\dfrac{1}{N}\sum_{\ell=1}^{N}f_{\ell}(\mathcal{P},\alpha_{\ell})}_{=f(\mathcal{P},\mathcal{A})},\label{eq:fed_dadil}
\end{align}
where $f_{\ell}$ is the objective function of each domain:
\begin{equation}
    f_{\ell}(\mathcal{P},\alpha_{\ell}) = \begin{cases}
        \mathcal{T}_{c}(\hat{Q}_{\ell},\mathcal{B}(\alpha_{\ell};\mathcal{P}))&  \ell\leq N-1,\\
        \mathcal{T}_{2}(\hat{Q}_{\ell},\mathcal{B}(\alpha_{\ell};\mathcal{P}))& \ell = N,\\
    \end{cases}\label{eq:local loss}
\end{equation}

While $\mathcal{P}$ are shared by all clients, the barycentric coordinates $\alpha_{\ell}$ are private to each client. Our federated strategy, presented in algorithm~\ref{alg:fed_dadil} is divided into two sub-routines: \emph{ClientUpdate} and \emph{ServerAggregate}.
\begin{algorithm}[ht]
\caption{\gls{feddadil}. The $N$ Clients are indexed by $\ell$. $n_{b}$ is the batch size, and $K$ the number of atoms.}
\begin{algorithmic}[1]
    \small
    \State Server initializes $\mathcal{P}_{g}^{(0)} = \{\hat{P}_{k}^{(0)}\}_{k=1}^{K}$
    \State clients initialize $\alpha_{\ell}^{(0)} \in \Delta_{K}$, $\forall \ell=1\cdots,N$
    \For{each round $r=1\cdots,R$}
        \For{client $\ell=1,\cdots,N$}
            \State Server communicates $\mathcal{P}_{g}^{(r)}$ to client $\ell$
            \State Initialize local dictionary $\mathcal{P}_{\ell}^{(0)} \leftarrow \mathcal{P}_{g}^{(r)}$
            \State $\mathcal{P}_{\ell}^{(r)} \leftarrow \text{ClientUpdate}(\mathcal{P}_{\ell}^{(0)},\alpha_{\ell}^{(r)})$
            \State client $\ell$ sends $\mathcal{P}_{\ell}^{(r)}$ to server.
        \EndFor
        \State $\mathcal{P}^{(r+1)}_{g} \leftarrow \text{ServerAggregate}(\{\mathcal{P}_{\ell}^{(r)}\}_{\ell=1}^{N})$
    \EndFor
\end{algorithmic}
\label{alg:fed_dadil}
\end{algorithm}

\noindent\textbf{Clients Update.} Similarly to FedAVG, at each communication round $r$, each client receives a global version from the server, noted as $\mathcal{P}_{g}^{(r)}$, which is copied into $\mathcal{P}_{\ell}^{(0)}$, the local version. The clients then proceed to optimize $(\mathcal{P}_{\ell}^{(0)}, \alpha_{\ell})$ through $E$ steps, by first splitting each $\hat{P}_{k}$ into $B = \lceil \nicefrac{n}{n_{b}} \rceil$ batches of size $n_{b}$. An epoch corresponds to an entire pass through the $B$ mini-batches. The loss is calculated between mini-batches of $\hat{P}_{k}$, and mini-batches of $\hat{Q}_{\ell}$. This is detailed in Algorithm~\ref{alg:client_update}. After each client step, it enforces $\alpha_{\ell} \in \Delta_{K}$ by projecting it orthogonally into the simplex.
\begin{algorithm}[ht]
\caption{ClientUpdate.}
\begin{algorithmic}[1]
    \small
    \For{local epoch $e=1,\cdots,E$}
        \For{batch $b=1,\cdots,B$}
            \State Let $\{ \{(\mathbf{z}_{i}^{(P_{k})},\mathbf{y}_{i}^{(P_{k})})\}_{i=b\times n_{b}}^{(b+1)\times n_{b}} \}_{k=1}^{K}$
            \State Sample $\{(\mathbf{z}_{i}^{(Q_{\ell})},\mathbf{y}_{i}^{(Q_{\ell})})\}_{i=1}^{n_{b}}$
            \State Compute loss $f_{\ell}(\mathcal{P}_{\ell}^{(e)},\alpha_{\ell})$
            \For{atom $k=1,\cdots,K$}
                \State $\mathbf{z}_{i}^{(P_{k})} \leftarrow \mathbf{z}_{i}^{(P_{k})} - \eta \nicefrac{\partial f_{\ell}}{\partial \mathbf{z}_{i}^{(P_{k})}}$
                \State $\mathbf{y}_{i}^{(P_{k})} \leftarrow \mathbf{y}_{i}^{(P_{k})} - \eta \nicefrac{\partial f_{\ell}}{\partial \mathbf{y}_{i}^{(P_{k})}}$
            \EndFor
            \State $\alpha_{\ell} \leftarrow \text{proj}_{\Delta_{K}}(\alpha_{\ell} - \eta \nicefrac{\partial f_{\ell}}{\partial \alpha_{\ell}})$
        \EndFor
    \EndFor
    \State Client sets $\alpha_{\ell}^{(r+1)} \leftarrow \alpha_{\ell}^{\star}$.
    \State \textbf{Return} $\mathcal{P}_{\ell}^{\star}$.
\end{algorithmic}
\label{alg:client_update}
\end{algorithm}

\noindent\textbf{Server Aggregation.} As the result of \emph{ClientUpdate}, at each communication round \gls{feddadil} has $N$ atom versions, $\{\mathcal{P}_{\ell}^{\star}\}_{\ell=1}^{N}$. As such, in the same way as \emph{FedAVG}~\cite{mcmahan2017communication}, one needs to \emph{aggregate} these different versions.  Given versions $\mathcal{P}_{0}$ and $\mathcal{P}_{1}$, we define the following arithmetic,
\begin{align*}
    \mathcal{P}_{0} + \alpha\mathcal{P}_{1} &:= \biggr{\{}\dfrac{1}{n}\sum_{i=1}^{n}\delta_{(\mathbf{z}_{i}^{(P_{k,0})} + \alpha \mathbf{z}_{i}^{(P_{k,1})}, \mathbf{y}_{i}^{(P_{k,0})} + \alpha \mathbf{y}_{i}^{(P_{k,1})})}\biggr{\}}_{k=1}^{K},
\end{align*}
i.e., we perform the summation w.r.t. the support of $\hat{P}_{k} \in \mathcal{P}$. The aggregation step is, therefore,
\begin{align}
    \mathcal{P}_{g}^{(r+1)} = \dfrac{1}{N}\sum_{\ell=1}^{N}\mathcal{P}_{\ell}^{\star}.\label{eq:atom_avg}
\end{align}
As we investigate in our experiments ($\S$~\ref{sec:experiments}), \gls{feddadil}'s objective behaves similarly to neural nets over interpolations of atom versions. Next, we present a novel result that shows that \gls{feddadil}'s objective behaves locally as a quadratic form,
\begin{theorem}\label{thm:dadil_loss}
    Let $(\mathcal{P},\mathcal{A})$ be a dictionary, and $\epsilon \in \mathbb{R}^{d}$ be a random perturbation. Let $\tilde{\mathcal{P}} = \{\tilde{P}_{k}\}_{k=1}^{K}$ such that,
    \begin{equation*}
        \tilde{P}_{k}(\mathbf{z},\mathbf{y}) = \dfrac{1}{n}\sum_{i=1}^{n}\delta((\mathbf{z},\mathbf{y}) - (\mathbf{z}_{j}^{(P_{k})}+\epsilon, \mathbf{y}_{j}^{(P_{k})})),
    \end{equation*}
    then,
    \begin{align*}
        f(\tilde{\mathcal{P}},\mathcal{A}) &= f(\mathcal{P},\mathcal{A}) + 2\epsilon^{T}\nabla_{x}f + \lVert \epsilon\rVert_{2}^{2},
    \end{align*}
\end{theorem}
A proof of this result is available on our supplementary materials.

\noindent\textbf{Complexity.} Since algorithm~\ref{alg:client_update} runs over mini-batches, the overall computational complexity on clients is cubic over the size of mini-batches $n_{b}$~\cite[$\S$ 4]{montesuma2023learning}. Furthermore, at each round, clients' communicate
\begin{align}
    |\mathcal{P}| = K \times n \times (d + n_{c})\label{eq:n_params}
\end{align}
floating-point numbers at each round. As we discuss in our experiments, while we cannot avoid communicating models in FedAVG, the \gls{dadil} step is comparatively lightweight.

\subsection{Domain Adaptation}

The dictionary learned at the end of Algorithm~\ref{alg:fed_dadil} \emph{models the distributional shift} occurring between sources and target domains. Unlike previous works on decentralized \gls{msda}, we do not align the sources with the target. We rather \emph{embrace} distributional shift by modeling it. To learn a classifier at the target domain we have two strategies, Reconstruction (R) or Ensembling (E)~\cite{montesuma2023learning}, which we now describe. Both methods stem from the fact that each $\hat{P}_{k}$ \textbf{has a labeled support}, i.e., $(\mathbf{Z}^{(P_{k})}, \mathbf{Y}^{(P_{k})})$. The following methods can thus be applied locally by the target client for learning a classifier that works on data following its probability distribution.

\noindent\textbf{Reconstructing the Target Domain.} Through \gls{feddadil}, we can express $\hat{Q}_{N} = \mathcal{B}(\alpha_{N};\mathcal{P})$. Let $(\mathbf{Z}^{(B_{N})}, \mathbf{Y}^{(B_{N})})$ denote the support of $\hat{B}_{N} = \mathcal{B}(\alpha_{N};\mathcal{P})$. These can be expressed in terms of the support of each $\hat{P}_{k}$, as,
\begin{equation}
    \mathbf{z}^{(B_{N})}_{i} = n\sum_{k=1}^{K}\alpha_{N,k}\sum_{j=1}^{n}\pi_{i,j}^{(k)}\mathbf{z}_{j}^{(P_{k})}\text{, and, }
    \mathbf{y}^{(B_{N})}_{i} = n\sum_{k=1}^{K}\alpha_{N,k}\sum_{j=1}^{n}\pi_{i,j}^{(k)}\mathbf{y}_{j}^{(P_{k})},\label{eq:dadil_r}
\end{equation}
where $\pi_{i,j}^{(k)}$ is the \gls{ot} plan between $\hat{B}_{N}$ and $\hat{P}_{k}$. We can fit a classifier directly on the target client with samples $\{(\mathbf{z}_{i}^{(B_{N})}, \mathbf{y}_{i}^{(B_{N})})\}_{i=1}^{n}$ with standard \gls{erm},
\begin{align*}
    \hat{\theta}_{R} = \argmin{\theta}\dfrac{1}{n}\sum_{i=1}^{n}\mathcal{L}(h(\mathbf{z}_{i}^{(Q_{N})};\theta),\mathbf{y}_{i}^{(Q_{N})})
\end{align*}
\noindent\textbf{Ensembling Atom Classifiers.} Conversely, the target domain can fit a classifier on each atom distribution, through,
\begin{align*}
    \hat{\theta}_{k} = \argmin{\theta}\dfrac{1}{n}\sum_{i=1}^{n}\mathcal{L}(h(\mathbf{z}_{i}^{(P_{k})}; \theta), \mathbf{y}_{i}^{(P_{k})}).
\end{align*}
One can then exploit the weights $\{\alpha_{N,k}\}_{k=1}^{K}$ to weight the predictions of classifiers $\{h(\cdot;\hat{\theta}_{k})\}_{k=1}^{K}$,
\begin{align}
    h_{E}(\mathbf{z}) = \sum_{k=1}^{K}\alpha_{N,k}h(\mathbf{z};\hat{\theta}_{k}).\label{eq:dadil_e}
\end{align}
Note that since $h(\mathbf{z}; \hat{\theta}_{k}) \in \Delta_{n_{c}}$, $h_{E}(\mathbf{z}) \in \Delta_{n_{c}}$, i.e., the result of eqn.~\ref{eq:dadil_e} is still a probability distribution over classes.


%% file: content/4_experiments.tex
\section{Experiments}\label{sec:experiments}

We provide a comparison of decentralized \gls{msda} methods on 5 visual adaptation benchmarks, namely: ImageCLEF~\cite{caputo2014imageclef}, Caltech-Office 10~\cite{gong2012geodesic}, Office 31~\cite{saenko2010adapting}, Office-Home~\cite{venkateswara2017deep} and Adaptiope~\cite{ringwald2021adaptiope}. We consider 3 methods from the decentralized \gls{msda} state-of-the-art, namely: FADA~\cite{peng2019federated}, KD3A~\cite{feng2021kd3a} and Co-MDA~\cite{liu2023co}. These methods were chosen due to their relevance, and availability of source code. Furthermore we consider adaptations of \gls{da} methods, such as $f$-DANN~\cite{ganin2016domain,peng2019federated} and $f-$WDGRL~\cite{shen2018wasserstein}. Further details on the benchmarks, and on the hyper-parameter settings of decentralized \gls{msda} methods are provided in the supplementary materials.

We compare \gls{feddadil} to other decentralized \gls{msda} strategies over 5 visual \gls{da} benchmarks. We present an overview of our results in table~\ref{tab:overview_benchmarks}. We use standard evaluation protocols in \gls{msda}, namely, we perform adaptation with a ResNet~\cite{he2016deep} backbone. The size of the backbone is selected to agree with previous research~\cite{peng2019federated,feng2021kd3a,liu2023co}. We run our experiments on a computer with a Ubuntu 22.04 OS, a 12th Gen Intel(R) Core\textsuperscript{TM} i9-12900H CPU with 64 GB of RAM, and with a NVIDIA RTX A100 GPU with 4GB of VRAM.



\begin{table}[ht]
    \centering
    \caption{Overview of Domain Adaptation benchmark}
    \begin{tabular}{llcccc}
        \toprule
        Benchmark & Backbone & \# Samples & \# Domains & \# Classes \\
        \midrule
        ImageCLEF & ResNet50 & 2400 & 4 & 12\\
        Caltech-Office 10 & ResNet101 & 2533 & 4 & 10\\
        Office31 & ResNet50 & 3287 & 3 & 31\\
        Office-Home & ResNet101 & 15500 & 4 & 65 \\
        Adaptiope & ResNet101 & 36900 & 3 & 123\\
        \bottomrule
    \end{tabular}
    \label{tab:overview_benchmarks}
\end{table}


In the following, we divide our experiments in six parts. First, we explore the optimization process of \gls{feddadil}. Second, we show empirically compare our method with prior art. Third we explore \gls{feddadil}'s performance with respect client parallelism. Fourth we explore the communication cost of our method. Fifth we analyze hyper-parameter sensitivity. Sixth we visualize the alignment between the target domain and its barycentric reconstruction.



\noindent\textbf{Decentralized Dataset Dictionary Learning.} Our empirical analysis is threefold: (i) how averaging different atom versions impacts our algorithm, (ii) visualizing the loss-landscape of $f(\mathcal{P},\mathcal{A})$ in a 2-D subspace and (iii) plotting \gls{feddadil}'s objective as a function of communication round. We illustrate our findings on the Caltech-Office 10 and Office 31 benchmarks.

First, given two versions $\mathcal{P}_{\ell}^{\star}$, $\ell=0, 1$, we define $\mathcal{P}_{t} = (1-t)\mathcal{P}_{0}^{\star} + t\mathcal{P}_{1}^{\star}$, $t \in [0, 1]$. Then, we proceed to evaluate $f(\mathcal{P}_{t},\mathcal{A})$ as a function of $t$, which is shown in Figure~\ref{fig:dadil_1d_loss}. Similarly to \emph{FedAVG}~\cite{mcmahan2017communication}, averaging different atom versions decreases the overall loss, which empirically validates our decentralized strategy.


Second, we explore Theorem 3.1. empirically. Let $(\mathcal{P}^{\star}, \mathcal{A}^{\star})$ be a dictionary. We define $(u, v)$ within $[-1.5, +1.5]^{2}$, and $(\mathcal{P}_{u}, \mathcal{P}_{v})$ such that, for $\epsilon \sim \mathcal{N}(0, \mathbf{I}_{d})$, $\mathbf{z}_{i}^{(P_{k,u})} := \epsilon$ and $\mathbf{y}_{i}^{(P_{k,u})} := \mathbf{y}_{i}^{(P_{k}^{\star})}$ (resp. $v$). Our analysis consists on visualizing,
\begin{align*}
    f(u, v) &= f(\mathcal{P}^{\star} + u\mathcal{P}_{u} + v\mathcal{P}_{v}, \mathcal{A}),
\end{align*}
i.e., randomly perturbing the features of the dictionary $\mathcal{P}^{\star}$, as in Theorem~\ref{thm:dadil_loss}. As shown in figure~\ref{fig:loss-landscape}, the loss has approximately quadratic level sets on the variables $(u, v)$.
\begin{figure}[ht]
    \centering
    \begin{subfigure}[t]{0.3\linewidth}
        \includegraphics[width=\linewidth]{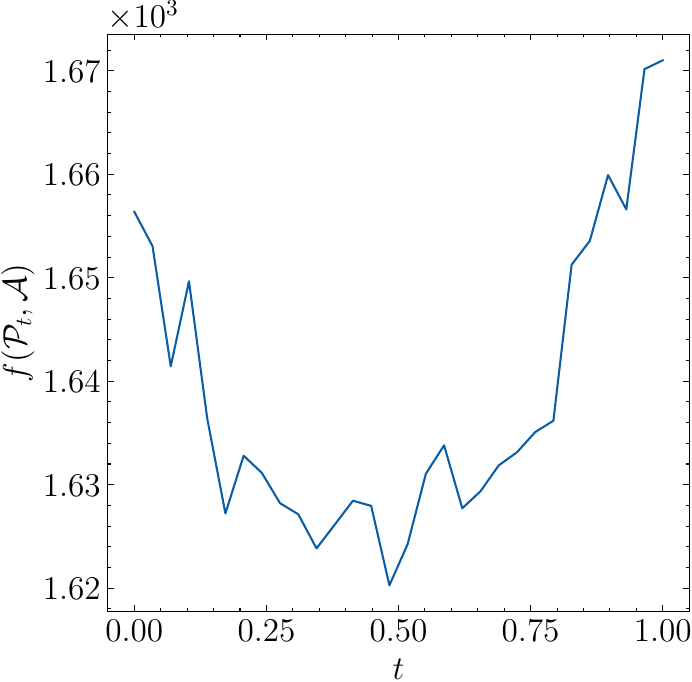}
        \caption{Interpolation of atom versions obtained by clients with same initialization.}
        \label{fig:dadil_1d_loss}
    \end{subfigure}\hfill
    \begin{subfigure}[t]{0.3\linewidth}
        \includegraphics[width=\linewidth]{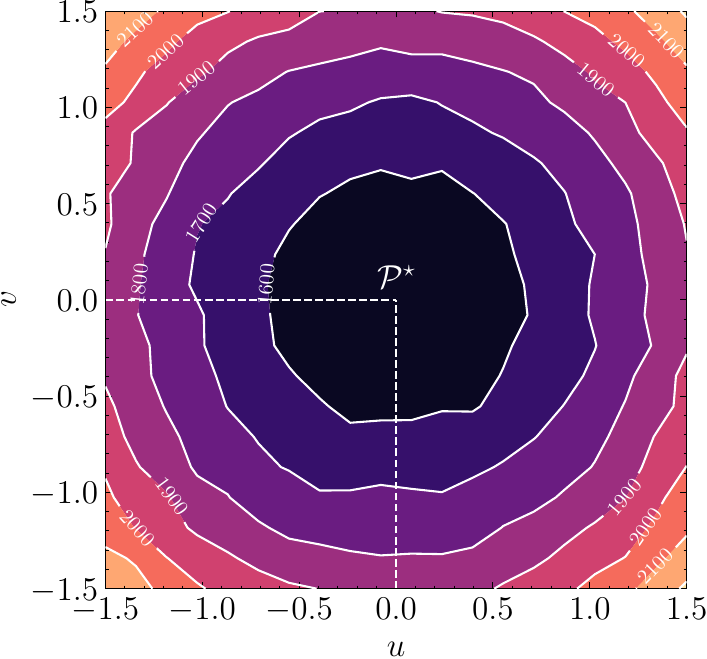}
        \caption{\gls{dadil}'s loss landscape. Colors represent loss values (marked by contour lines), whereas $(u,v)$ parametrize perturbations.}
        \label{fig:dadil_2d_loss}
    \end{subfigure}\hfill
    \begin{subfigure}[t]{0.3\linewidth}
        \includegraphics[width=\linewidth]{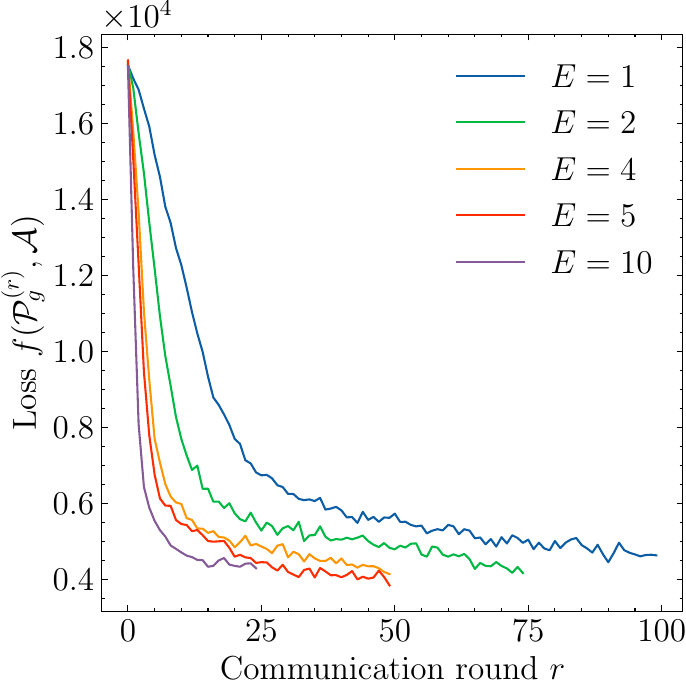}
        \caption{\gls{dil} loss function as a function of communication rounds for various values of local iterations $E$. Benchmark: Office 31, target: amazon}
        \label{fig:dil_loss_n_client_it}
    \end{subfigure}
    \caption{Analysis $1-$dimensional (a) and $2-$dimensional (b) of \gls{dadil}'s loss. (a) Similarly to \emph{FedAVG}~\cite{mcmahan2017communication}, interpolating between two atom versions obtained by clients with a shared initialization decreases the overall loss value. (b) We illustrate Theorem 3.1. empirically on Caltech-Office 10, showing that \gls{dadil}'s loss is locally quadratic.}
    \label{fig:loss-landscape}
\end{figure}
Finally, we analyze how the number of local iterations executed by the clients, $E$, impacts our federated \gls{dil} strategy. At one hand, for an increasing $E$, one parallelizes the \gls{dil} problem, as more iterations are done within clients before averaging the multiple atom versions. At the other hand, if the clients perform many local steps, they risk to overfit the atoms to their own data. Overall, for a wide range of values for $E$, we verify that \gls{dil}'s optimization converges towards a local minima, as shown in figure~\ref{fig:dil_loss_n_client_it}. This illustrate an advantage w.r.t. other decentralized \gls{msda} methods, such as KD3A~\cite{feng2021kd3a}, who fix $E=1$ in their experiments. We provide further analysis w.r.t. decentralized \gls{msda} performance in the next section.


\begin{table}[ht]
    \centering
    \caption{Experimental Results on decentralized \gls{msda} benchmarks. $\dagger$, $\ddag$ and $\star$ indicates results from~\cite{peng2019federated},~\cite{feng2021kd3a} and~\cite{liu2023co} respectively. $\uparrow$ denotes that higher is better. Additional details on our results are given in our supplementary materials.}
    \begin{subtable}{0.31\linewidth}
        \resizebox{\linewidth}{!}{\begin{tabular}{lcccc>{\columncolor[gray]{0.9}}c}
            \toprule
            Algorithm & \scriptsize{Amazon} & \scriptsize{dSLR} & \scriptsize{Webcam} & \scriptsize{Caltech} & \scriptsize{Avg. $\uparrow$} \\
            \midrule
            FedAVG & 86.1 & 98.3 & 99.0 & 87.8 & 92.8\\
            FedProx & 96.9 & 97.2 & 100.0 & 92.5 & 96.6\\
            \midrule
            $f$-DANN$^{\dagger}$ & 83.4 & 85.9 & 87.1 & 88.5 & 86.3 \\
            $f$-WDGRL & 97.9 & 97.1 & 100.0 & 95.6 & 97.6 \\
            FADA$^{\dagger}$ & 84.2 & 87.1 & 88.1 & 88.7 & 87.1 \\
            KD3A$^{\ddag}$ & 97.4 & 98.4 & 99.7 &  \underline{96.4} & 97.9 \\
            Co-MDA$^{\star}$ & \underline{98.2} & \textbf{100.0} & \textbf{100.0} & \textbf{96.9} & \textbf{98.8}\\
            \midrule
            FedDaDiL-E & \textbf{99.0} & \textbf{100.0} & \textbf{100.0} & 96.1 & \underline{98.7} \\
            FedDaDiL-R & \textbf{99.0} & \textbf{100.0} & \textbf{100.0} & 95.6 & 98.6 \\
            \bottomrule
        \end{tabular}}
        \caption{Caltech-Office 10.}
    \end{subtable}\hfill
    \begin{subtable}{0.31\linewidth}
        \resizebox{\linewidth}{!}{\begin{tabular}{lcccc>{\columncolor[gray]{0.9}}c}
            \toprule
            Algorithm & \scriptsize{Caltech} & \scriptsize{Bing} & \scriptsize{ImageNet} & \scriptsize{Pascal} & \scriptsize{Avg. $\uparrow$} \\
            \midrule
            FedAVG & 96.7 & \underline{65.8} & 94.2 & 77.5 & \underline{83.6}\\
            FedProx & \underline{96.7} & \underline{65.8} & 93.3 & 76.7 & 83.1\\
            \midrule
            $f$-DANN & 96.7 & 64.2 & 87.5 & \underline{80.0} & 82.1 \\
            $f$-WDGRL & 92.5 & 63.3 & 86.7 & 74.2 & 79.2 \\
            FADA & 95.0 & 64.2 & 90.0 & 74.2 & 80.9 \\
            KD3A & 93.3 & \textbf{69.2} & \textbf{95.5} & 73.3 & 82.8 \\
            Co-MDA & 94.2 & 65.0 & 91.5 & 78.0 & 82.2 \\
            \midrule
            FedDaDiL-E & \textbf{98.3} & \textbf{69.2} & {93.3} & \textbf{81.6} & \textbf{85.6}\\
            FedDaDiL-R & \textbf{98.3} & \textbf{69.2} & \underline{95.0} & \underline{80.0} & \textbf{85.6}\\
            \bottomrule
        \end{tabular}}
        \caption{ImageCLEF.}
    \end{subtable}\hfill
    \begin{subtable}{0.28\linewidth}
        \resizebox{\linewidth}{!}{\begin{tabular}{lccc>{\columncolor[gray]{0.9}}c}
            \toprule
            Algorithm & \scriptsize{Amazon} & \scriptsize{dSLR} & \scriptsize{Webcam} & \scriptsize{Avg. $\uparrow$} \\
            \midrule
            FedAVG & 67.5 & 95.0 & 96.8 & 86.4\\
            FedProx & 67.4 & 96.0 & 96.8 & 86.7\\
            \midrule
            $f$-DANN & 67.7 & 99.0 & 95.6 & 87.4 \\
            $f$-WDGRL & 64.8 & 99.0 & 94.9 & 86.2 \\
            FADA & 62.5 & 97.0 & 93.7 & 84.4 \\
            KD3A & 65.2 & \textbf{100.0} & \underline{98.7} & 88.0 \\
            Co-MDA & 64.8 & \underline{99.8} & \underline{98.7} & 87.8\\
            \midrule
            FedDaDiL-E & \textbf{71.2} & \textbf{100.0} & 98.2 & \underline{89.8} \\
            FedDaDiL-R & \underline{70.6} & \textbf{100.0} & \textbf{99.4} & \textbf{90.0} \\
            \bottomrule
        \end{tabular}}
        \caption{Office 31.}
    \end{subtable}\\
    \begin{subtable}{0.31\linewidth}
        \resizebox{\linewidth}{!}{\begin{tabular}{lcccc>{\columncolor[gray]{0.9}}c}
        \toprule
        Algorithm & \scriptsize{Art} & \scriptsize{Clipart} & \scriptsize{Product} & \scriptsize{Real-World} & \scriptsize{Avg. $\uparrow$} \\
        \midrule
        FedAVG & 72.9 & 62.2 & 83.7 & 85.0 & 76.0\\
        FedProx & 70.8 & 63.7 & 83.6 & 83.1 & 75.3\\
        \midrule
        $f$-DANN & 70.2 & \underline{65.1} & 84.8 & 84.0 & 76.0 \\
        $f$-WDGRL & 68.2 & 64.1 & 81.3 & 82.5 & 74.0 \\
        FADA & - & - & - & - & - \\
        KD3A & 73.8 & 63.1 & 84.3 & 83.5 & 76.2 \\
        Co-MDA$^{\star}$ & 74.4 & 64.0 & 85.3 & 83.9 & 76.9\\
        \midrule
        FedDaDiL-E & \underline{75.7} & {64.7} & \textbf{85.9} & \textbf{85.6} & \textbf{78.0} \\
        FedDaDiL-R & \underline{76.5} & \textbf{65.2} & \textbf{85.9} & \underline{84.2} & \underline{78.0} \\
        \bottomrule
    \end{tabular}}
        \caption{Office-Home.}
    \end{subtable}\hspace{2cm}
    \begin{subtable}{0.28\linewidth}
        \resizebox{\linewidth}{!}{\begin{tabular}{lccc>{\columncolor[gray]{0.9}}c}
        \toprule
        Algorithm & \scriptsize{Synthetic} & \scriptsize{Real} & \scriptsize{Product} & \scriptsize{Avg. $\uparrow$} \\
        \midrule
        FedAVG & 41.3 & 73.7 & 89.3 & 68.1\\
        FedProx & 38.5 & 70.6 & 87.9 & 65.7\\
        \midrule
        $f$-DANN & 42.2 & 71.6 & 89.1 & 67.4 \\
        $f$-WDGRL & 34.7 & 64.3 & 84.4 & 61.1 \\
        FADA & - & - & - & - \\
        KD3A & 49.6 & \textbf{83.3} & \textbf{92.1} & 75.0 \\
        Co-MDA & 39.3 & \underline{81.0} & 89.0 & 69.7\\
        \midrule
        FedDaDiL-E & 62.2 & 74.1 & \underline{91.1} & \textbf{75.8} \\
        FedDaDiL-R & \textbf{62.3} & 74.4 & 90.6 & \underline{75.4} \\
        \bottomrule
    \end{tabular}}
        \caption{Adaptiope.}
    \end{subtable}
    \label{tab:fed_da_results}
\end{table}

\noindent\textbf{Decentralized Visual Domain Adaptation.} In the following we refer to table~\ref{tab:fed_da_results}, where we summarize our comparisons. First, \emph{FedAVG} is a strong baseline when there is a limited amount of data. Indeed, on small benchmarks such as ImageCLEF and Office31, it performs better or equally than existing decentralized \gls{msda} algorithms. This is not true on larger benchmarks, such as Adaptatiope.

Second, non-i.i.d. federated learning methods, such as \emph{FedProx}~\cite{li2020federated} are approximately equivalent to \emph{FedAVG}. While decentralized \gls{msda} is also concerned with learning under non-i.i.d. data, the fact that these methods do not leverage target domain data hinders their performance

Third, our methods, \gls{feddadil}-E and R are able to outperform \emph{FedAVG} and other decentralized \gls{msda} benchmarks \textbf{without aligning domains}. Indeed, \gls{feddadil} works in a fundamentally different way than existing decentralized \gls{msda} algorithms, as it \emph{embraces distributional shift}. The adaptation is done by \emph{reconstructing} information on the target domain (\gls{feddadil}-R, cf. eq.~\ref{eq:dadil_r}), or by weighting predictions of atom distributions (\gls{feddadil}-E, cf. eq.~\ref{eq:dadil_e}). 

\noindent\textbf{Performance under Client Parallelism.} We analyze the performance of \gls{feddadil} as we increase the degree of client parallelism, which is a function of the number of local iterations $E$. As we verified in the last section, for an increasing $E$ \gls{feddadil} converges faster w.r.t. the rounds of communication $r$. In table~\ref{tab:client_parallelism}, we explore the performance of \gls{feddadil}-R and E in comparison with KD3A~\cite{feng2021kd3a} and Co-MDA~\cite{liu2023co} w.r.t. $E$. Note that, in these previous works, the authors fixed $E = 1$ in their experiments. They further verified a degradation in performance for an increasing $E$.


As we verify empirically in table~\ref{tab:client_parallelism}, there is indeed a degradation in KD3A and Co-MDA with an increasing $E$. Nonetheless, \gls{feddadil}-R and E keep approximately the same level of performance. As a result, on top of \emph{outperforming} current SOTA, our method is more resilient to parallelism.
\begin{table}[ht]
    \centering
    \caption{Adaptation performance w.r.t. client parallelism on Office 31 benchmark.}
    \begin{tabular}{lccccc}
        \toprule
         Method & $E=1$ & $E=2$ & $E=4$ & $E=5$ & $E=10$ \\
         \midrule
         KD3A & 87.8 & 86.7 & 86.0 & 85.4 & 65.9 \\
         CoMDA & 88.0 & 86.6 & 86.7 & 86.6 & 84.2 \\
         DaDiL-E & 89.9 & 89.9 & 89.2  & 89.8  & 88.9 \\
         DaDiL-R & 89.8 & 90.4 & 89.4 & 90.0 & 89.4 \\
         \bottomrule
    \end{tabular}
    \label{tab:client_parallelism}
\end{table}

\noindent\textbf{Communication Cost.} We compare the communication cost in bits at each round, between \gls{feddadil} and conventional decentralized \gls{msda} methods, such as FADA~\cite{peng2019federated} and KD3A~\cite{feng2021kd3a}. To do so, we calculate the total number of parameters in \gls{feddadil}, which, as we discussed in section~\ref{sec:fed_dadil}, corresponds to $|\mathcal{P}| = K \times n \times (d + n_{c})$. Further details on the choice of $K, n$ and $n_{b}$ are given in our supplementary materials. We then calculate the number of bits used to communicate these parameters, using 32-bit floating point precision, and divide it by the number of bits used to encode the backbone network, using the same precision. Our results are summarized in figure~\ref{fig:comm_cost}.
\begin{figure}[ht]
\centering
\begin{subfigure}{0.3\linewidth}
    \includegraphics[width=\linewidth]{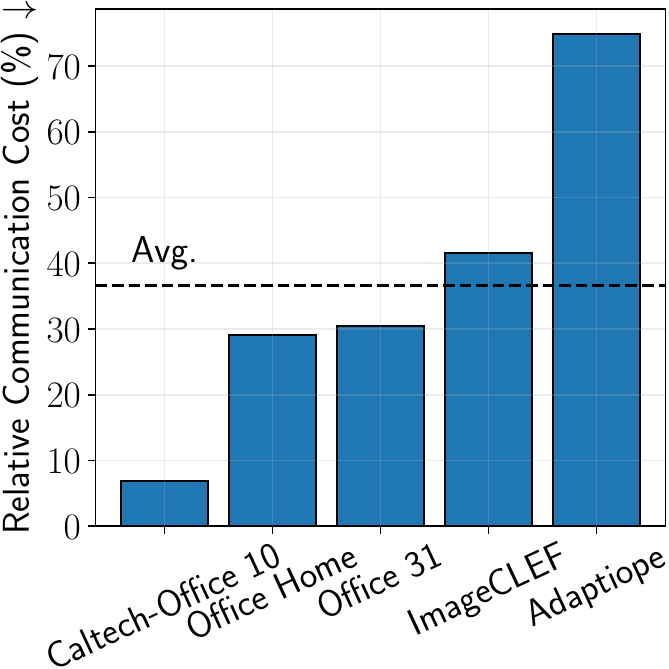}
    \caption{Communication cost.}
    \label{fig:comm_cost}
\end{subfigure}\hfill
\begin{subfigure}{0.35\linewidth}
    \includegraphics[width=\linewidth]{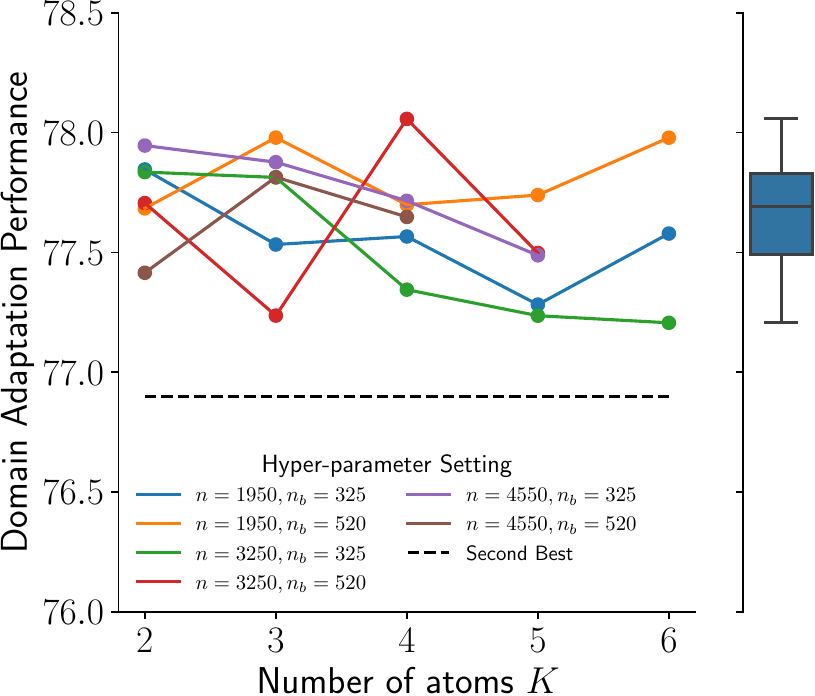}
    \caption{Hyper-parameter sensitivity.}
    \label{fig:hyperparam}
\end{subfigure}
\caption{In (a), we show the communication cost in \% relative to the cost of communicating the parameters of the backbone. In (b), we show the hyper-parameter sensitivity of \gls{dadil} on the Office-Home benchmark.}
\end{figure}

With respect figure~\ref{fig:comm_cost}, while we cannot avoid communicating the whole networks during step 1 (\emph{FedAVG}), the \gls{dadil} step communicates much less parameters than the network used to encode its inputs. As a result, taking into account table~\ref{tab:client_parallelism} and figure~\ref{fig:comm_cost}, the \gls{dadil} step has better communication efficiency than previous decentralized \gls{msda} methods.

\noindent\textbf{Parameter Sensitivity.} Here, we evaluate the sensitivity of \gls{feddadil} adaptation performance with respect its hyper-parameters, namely, number of samples $n$, batch size $n_{b}$ and number of atoms $K$. On total, this generates $36$ possible models, whose performance is shown in figure~\ref{fig:hyperparam}. Note that, over the chosen range of hyper-parameters, \gls{feddadil} has a performance of $77.6 \pm 0.22$, well above the second-best method (KD3A~\cite{feng2021kd3a}, with $76.9\%$ of domain adaptation performance). This shows that, overall, our proposed algorithm is robust with respect the choice of its hyper-parameters. We provide the complete list of hyper-parameters (over all benchmarks) in the supplementary materials. Furthermore, $n$ and $K$ control the complexity of our dictionary. From figure~\ref{fig:hyperparam}, we see that a small number of atoms (e.g., $K \leq 3$) and samples (e.g., $n=1950$) yields the best results. In these cases, we are strictly reducing the total number of samples, as $1950 \times 3  = 5850 < 15500$ in the Office-Home benchmark.

\noindent\textbf{Feature Visualization.} We compare the alignment of distributions with our method and those of \gls{kd3a} and \gls{fada}. In our case, we align the target with $\mathcal{B}(\alpha_{N};\mathcal{P})$, whereas other methods align the target with the source domains. We visualize these alignments by embedding the samples in $\mathbb{R}^{2}$ through t-SNE~\cite{van2008visualizing} (c.f., figure~\ref{fig:domain_alignment}). While the alignment towards domains Webcam and dSLR works well, aligning the sources with Amazon is more challenging. In this case \gls{feddadil} manages to reconstruct it with dictionary learning, which explains its superior performance.

\noindent\textbf{Dataset Distillation.} We explore the use of \gls{dadil} for dataset distillation~\cite{sachdeva2023data}, i.e., creating a reduced summary for a given dataset. We do so in an \gls{fda} setting, i.e., without labeled sampels in the target domain being summarized. As such, we analyze $\hat{B}_{T}$ as a function of $n_{B} = n_{c} \times \text{SPC}$, for $\text{SPC} \in \{1, 10, 20\}$.

In figure~\ref{fig:tsne_reconstruction} (a-c), we analyze points in $\hat{B}_{T}$ in comparison with $\hat{Q}_{T}$ through t-SNE. We use the synthetic domain in Adaptiope as target domain. As SPC grows, one captures progressively better the target distribution. Furthermore, the confidence of reconstructed labels, measured through the entropy, increases with SPC, as shown in figures~\ref{fig:tsne_reconstruction} (d-f) and (g-i). These results agree with previous research involving distillation and Wasserstein barycenters~\cite{montesuma2024multi}.

Finally, we analyze the performance of summaries created through \gls{dadil} in figures~\ref{fig:tsne_reconstruction} (j-l), in comparison with random sampling the sources and target domain. \gls{dadil} summaries' performance increases rapidly with \gls{spc}, but becomes constant as we increase the amount of data. As such, \gls{dadil} achieves novel performance with a summary with $1.67\%$ of the total amount of samples.
\clearpage

\begin{figure}[ht]
\centering
\begin{subfigure}[t]{\dimexpr0.23\textwidth+20pt\relax}
    \makebox[20pt]{\raisebox{40pt}{\rotatebox[origin=c]{90}{Amazon}}}%
    \includegraphics[width=\dimexpr\linewidth-20pt\relax]{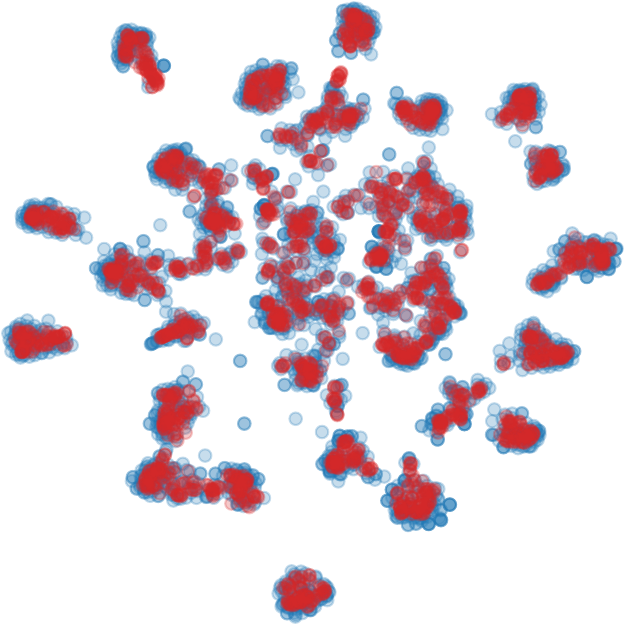}
    \makebox[20pt]{\raisebox{40pt}{\rotatebox[origin=c]{90}{dSLR}}}%
    \includegraphics[width=\dimexpr\linewidth-20pt\relax]{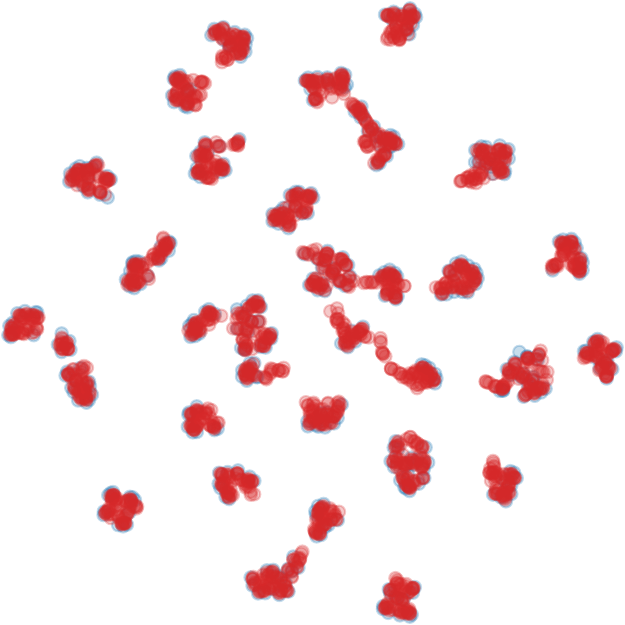}
    \makebox[20pt]{\raisebox{40pt}{\rotatebox[origin=c]{90}{Webcam}}}%
    \includegraphics[width=\dimexpr\linewidth-20pt\relax]{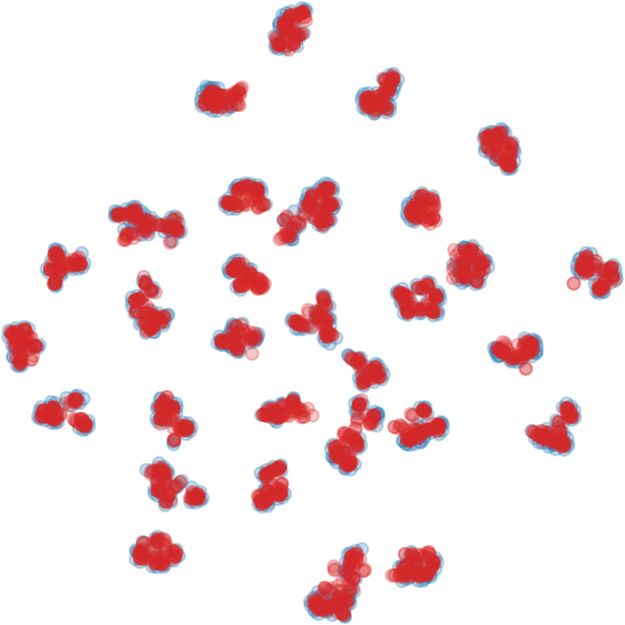}
    \caption{D$^{2}$adil}
\end{subfigure}\hfill
\begin{subfigure}[t]{0.23\textwidth}
    \includegraphics[width=\textwidth]{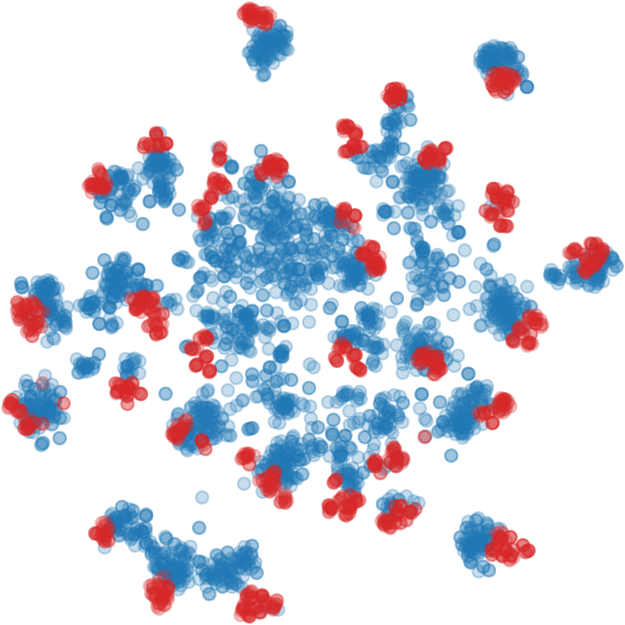}
    \includegraphics[width=\textwidth]{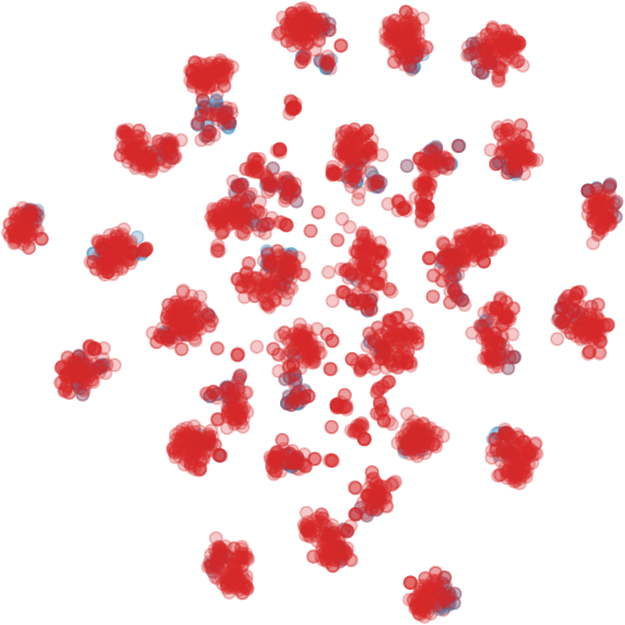}
    \includegraphics[width=\textwidth]{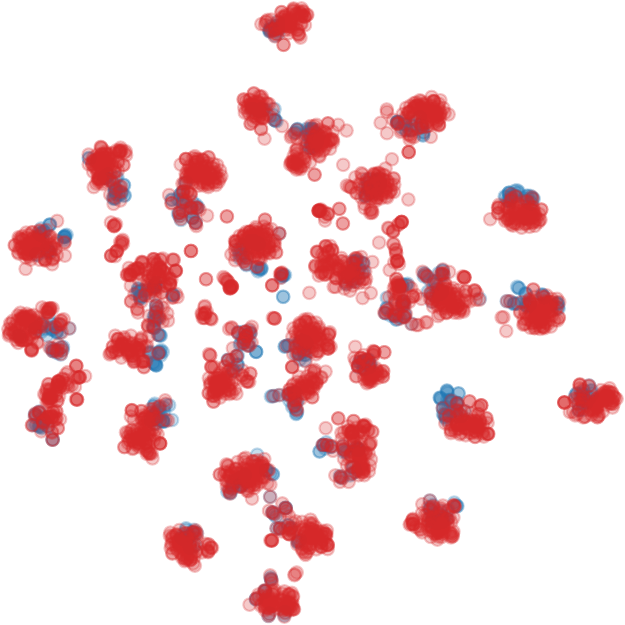}
    \caption{KD3A}
\end{subfigure}\hfill
\begin{subfigure}[t]{0.23\textwidth}
    \includegraphics[width=\textwidth]{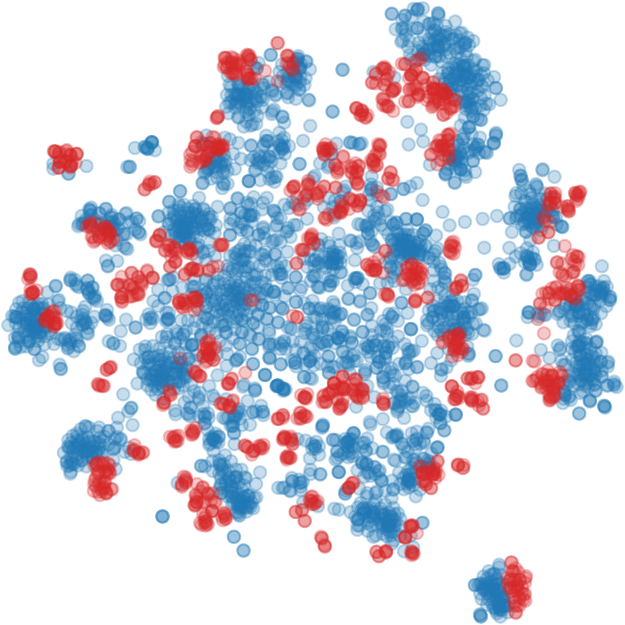}
    \includegraphics[width=\textwidth]{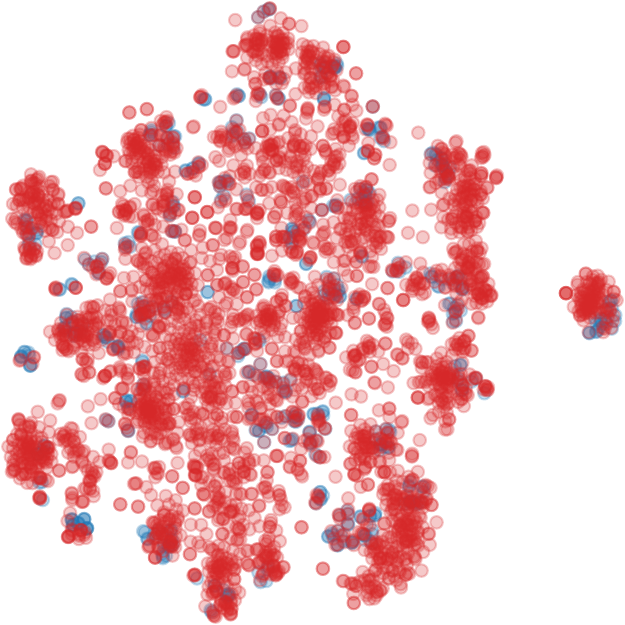}
    \includegraphics[width=\textwidth]{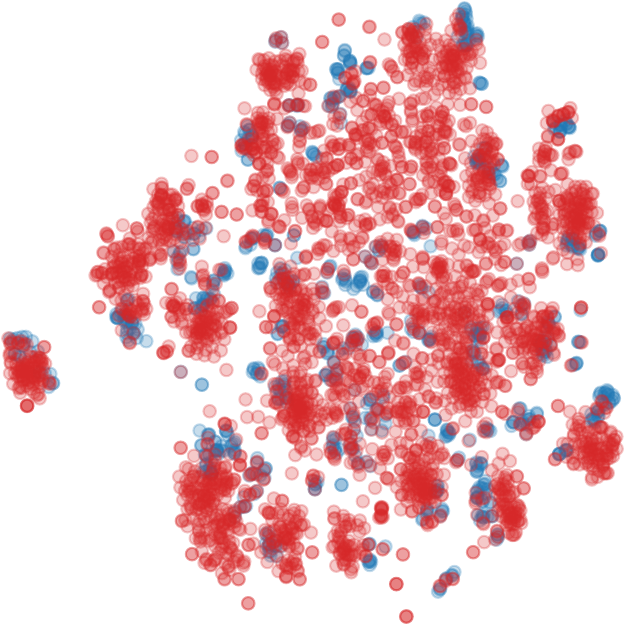}
    \caption{FADA}
\end{subfigure}
\caption{t-SNE embeddings of distribution alignments of \gls{feddadil}, KD3A and FADA. Blue points correspond to target domain points, whereas red points correspond to samples in the barycenter support (\gls{feddadil}) and source domains (KD3A and FADA).}
\label{fig:domain_alignment}
\end{figure}



\begin{figure}[H]
    \centering
    \begin{subfigure}{0.28\linewidth}
        \includegraphics[width=\linewidth]{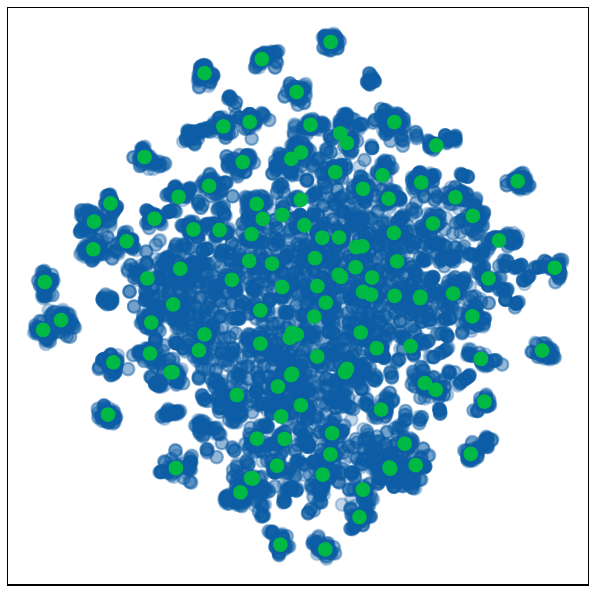}
        \caption{SPC=$1$}
    \end{subfigure}
    \begin{subfigure}{0.28\linewidth}
        \includegraphics[width=\linewidth]{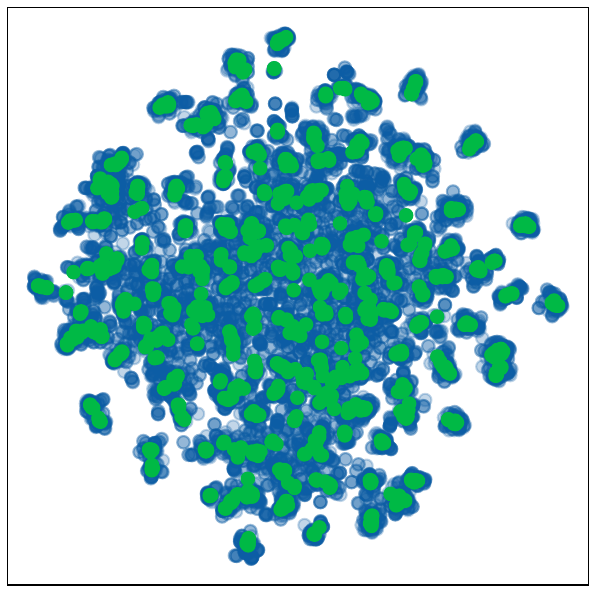}
        \caption{SPC=$10$}
    \end{subfigure}
    \begin{subfigure}{0.28\linewidth}
        \includegraphics[width=\linewidth]{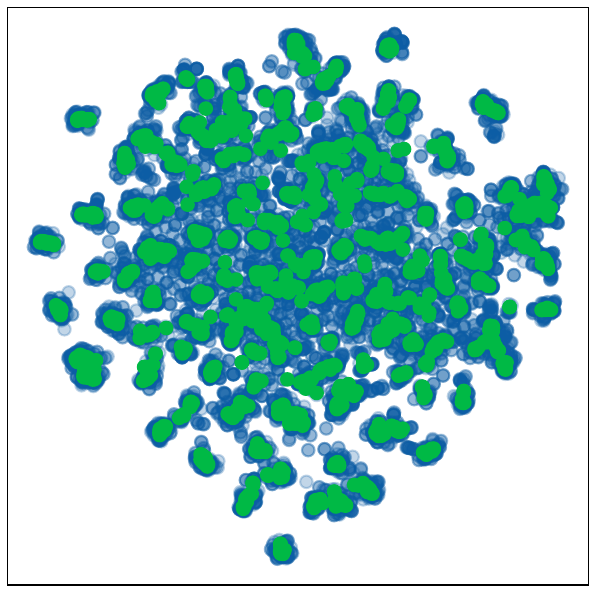}
        \caption{SPC=$20$}
    \end{subfigure}\\
    \begin{subfigure}{0.28\linewidth}
        \includegraphics[width=\linewidth]{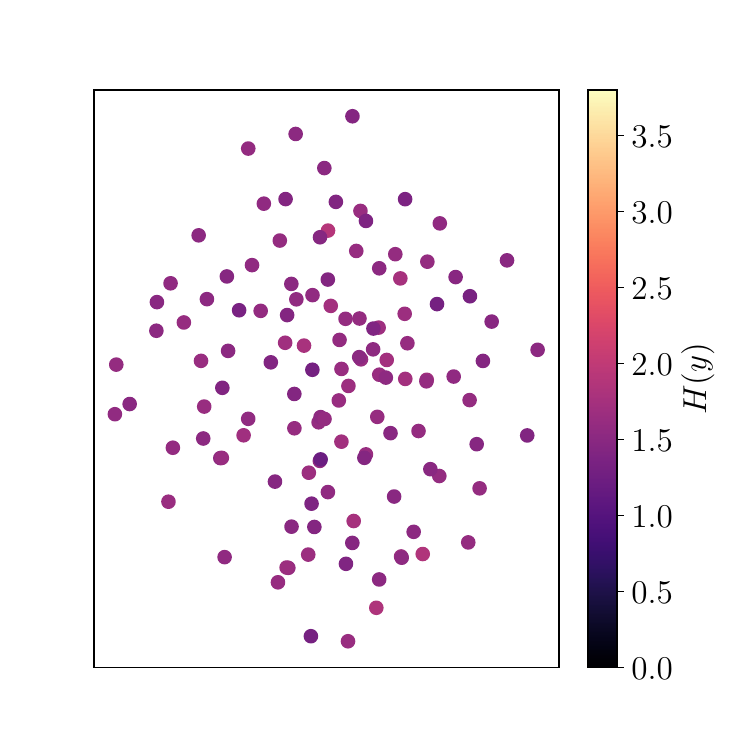}
        \caption{SPC=$1$}
    \end{subfigure}
    \begin{subfigure}{0.28\linewidth}
        \includegraphics[width=\linewidth]{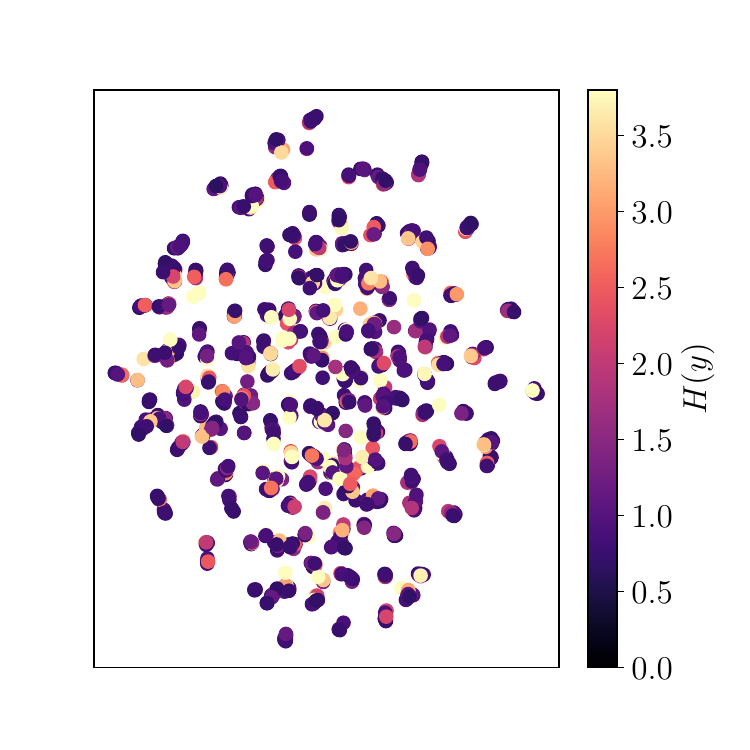}
        \caption{SPC=$10$}
    \end{subfigure}
    \begin{subfigure}{0.28\linewidth}
        \includegraphics[width=\linewidth]{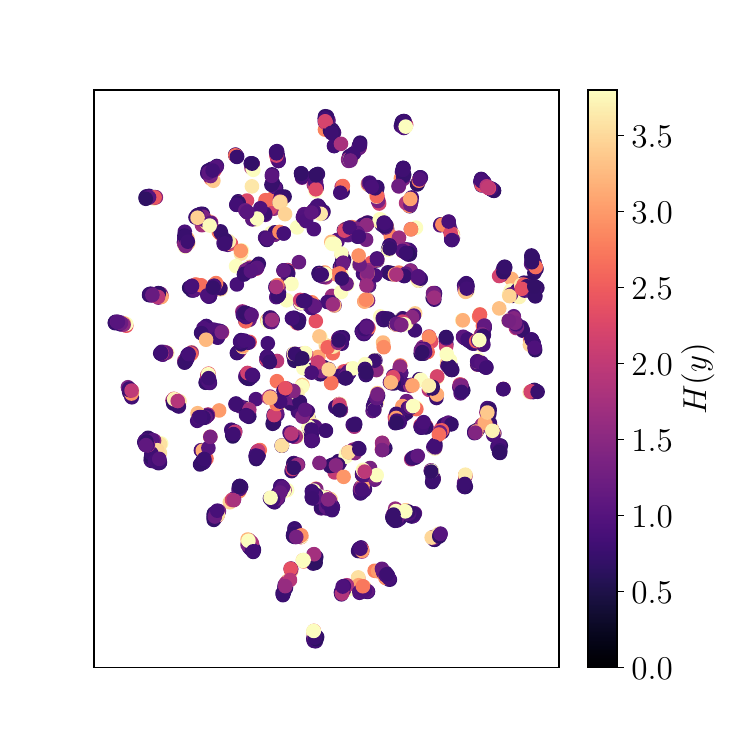}
        \caption{SPC=$20$}
    \end{subfigure}\\
    \begin{subfigure}{0.28\linewidth}
        \includegraphics[width=\linewidth]{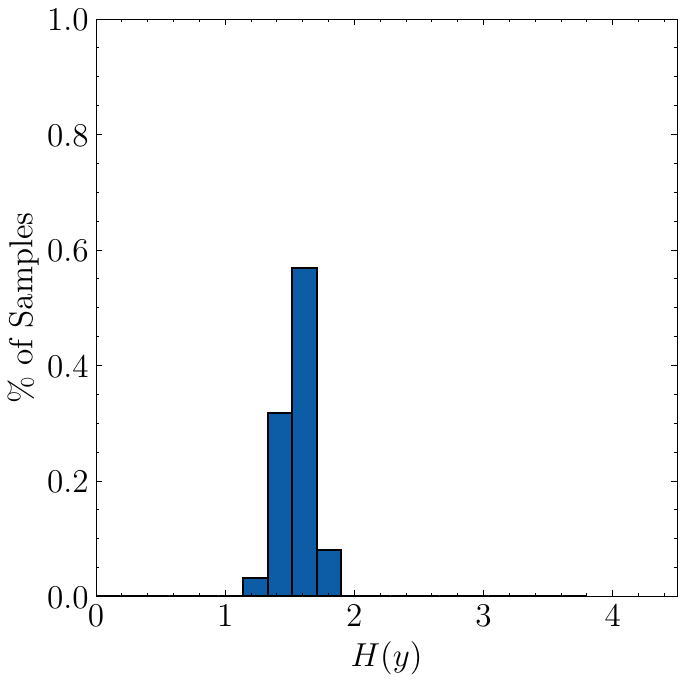}
        \caption{SPC=$1$}
    \end{subfigure}\hfill
    \begin{subfigure}{0.28\linewidth}
        \includegraphics[width=\linewidth]{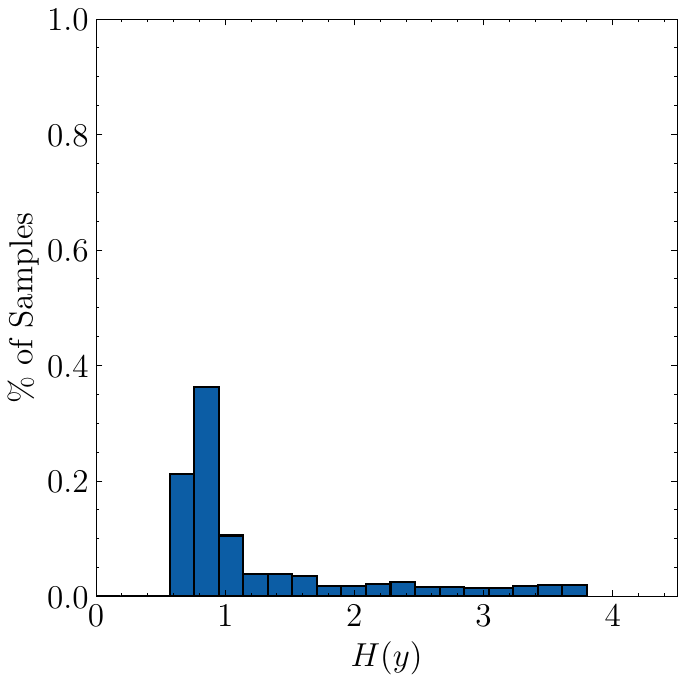}
        \caption{SPC=$10$}
    \end{subfigure}\hfill
    \begin{subfigure}{0.28\linewidth}
        \includegraphics[width=\linewidth]{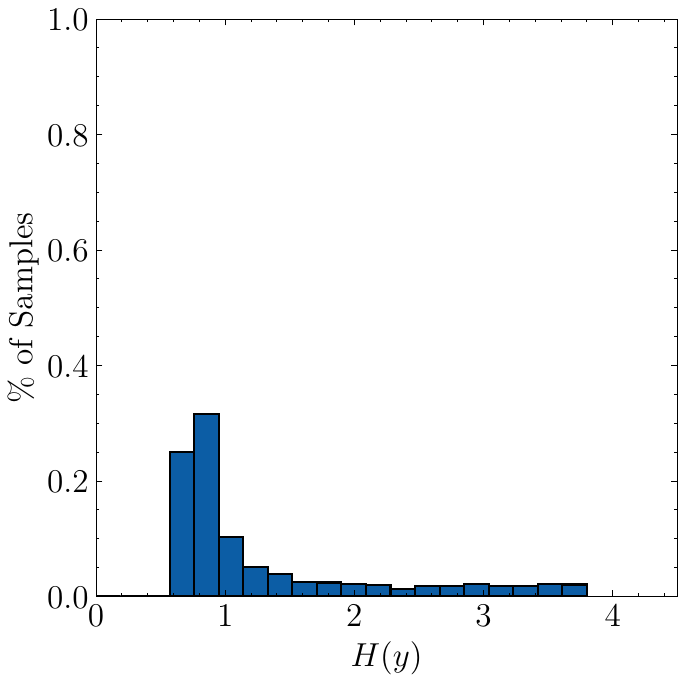}
        \caption{SPC=$20$}
    \end{subfigure}\\
    \begin{subfigure}{0.28\linewidth}
        \includegraphics[width=\linewidth]{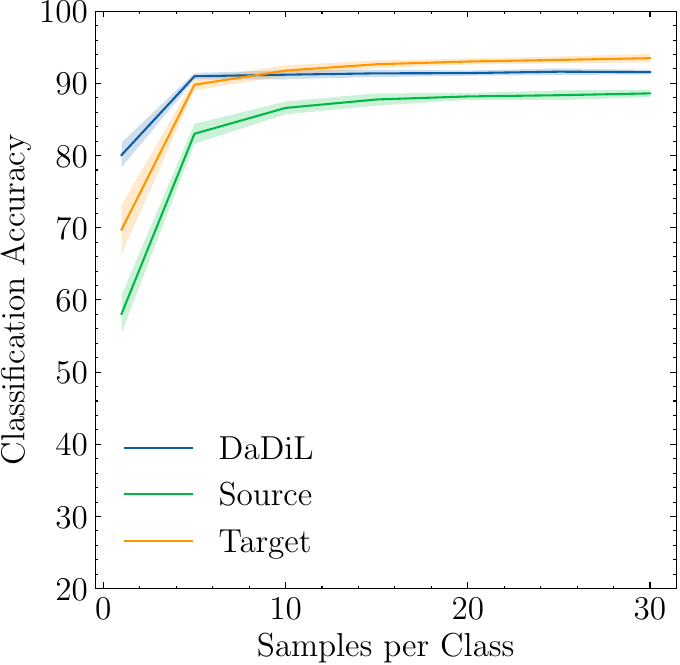}
        \caption{Target: Product}
    \end{subfigure}\hfill
    \begin{subfigure}{0.28\linewidth}
        \includegraphics[width=\linewidth]{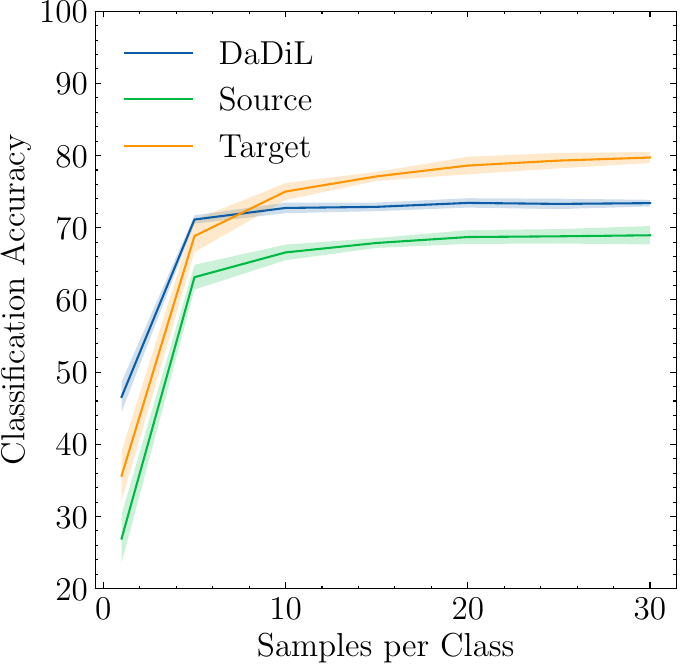}
        \caption{Target: Real}
    \end{subfigure}\hfill
    \begin{subfigure}{0.28\linewidth}
        \includegraphics[width=\linewidth]{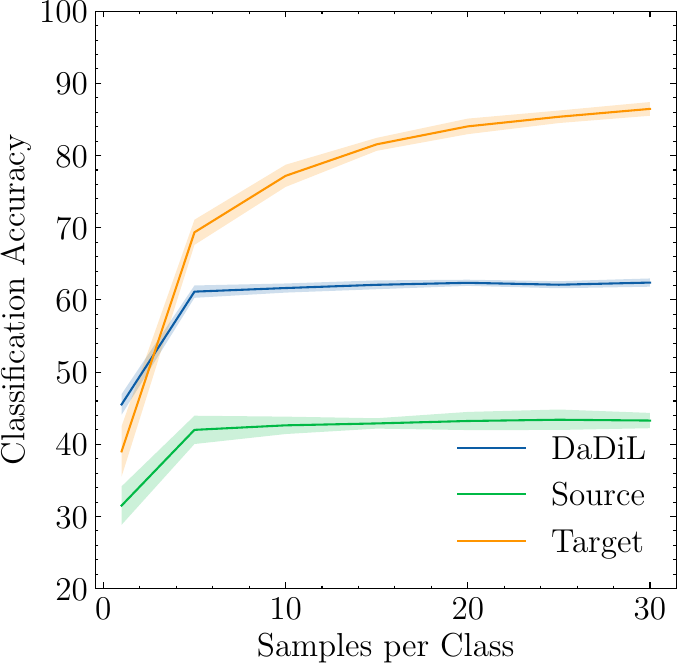}
        \caption{Target: Synthetic}
    \end{subfigure}
    \caption{Dataset distillation on Adaptiope benchmark. (a-c) show a comparison between real (blue) and reconstructed data points (green). (d-f) show the entropy of labels of reconstructed data points. For low values of \gls{spc}, samples have higher label entropy. (g-i) show the distribution of label entropies, in line with the conclusion of (d-f). Finally, (j-l) compares the performance of distillation with samples generated by \gls{dadil}, in comparison to random sub-sampling the source (green) and target (yellow). For \gls{spc}$=5$, one reaches state-of-the-art performance. This represents around 1.67\% of the total amount of samples}
    \label{fig:tsne_reconstruction}
\end{figure}


%% file: content/5_conclusion.tex
\section{Conclusion}\label{sec:conclusion}

We propose a novel federated algorithm for learning dictionaries of empirical distributions for federated domain adaptation. The main idea of our approach is keeping server atoms public, whereas clients' barycentric coordinates are private. Our strategy is based on two steps. First, one leans a neural net encoder through standard \emph{FedAVG}~\cite{mcmahan2017communication}. Second, we rethink the \gls{dadil} strategy~\cite{montesuma2023learning} in a federated setting. Our \emph{end-to-end} decentralized \gls{da} strategy improves adaptation performance on 5 visual \gls{da} benchmarks (table~\ref{tab:fed_da_results}). On top of that, we show that our strategy handles \emph{client parallelism} better than previous works (figure~\ref{fig:dil_loss_n_client_it} and table~\ref{tab:client_parallelism}), while being relatively lightweight in comparison with communication deep neural nets' parameters (figure~\ref{fig:comm_cost}). We further show that our method is robust with respect its hyper-parameters (figure~\ref{fig:hyperparam}), and that the alignment between the target domain and its reconstruction is better than with source domains (figure~\ref{fig:domain_alignment}).


